\documentclass[letterpaper]{article} 
\usepackage{aaai2026}  
\usepackage{times}  
\usepackage{helvet}  
\usepackage{courier}  
\usepackage[hyphens]{url}  
\usepackage{graphicx} 
\usepackage{natbib}  
\usepackage{caption} 
\usepackage{algorithm}
\usepackage{algorithmic}
\usepackage{amsmath}
\usepackage{amssymb}
\usepackage{mathtools}
\usepackage{amsthm}
\usepackage{enumitem}
\theoremstyle{plain}
\usepackage{tikz}
\usepackage{subcaption}
\usepackage{booktabs}
\usepackage{placeins}
\newtheorem{theorem}{Theorem}[section]
\newtheorem{example}[theorem]{Example}
\newtheorem{proposition}[theorem]{Proposition}
\newtheorem{lemma}[theorem]{Lemma}
\newtheorem{corollary}[theorem]{Corollary}
\theoremstyle{definition}
\newtheorem{definition}[theorem]{Definition}
\newtheorem{assumption}[theorem]{Assumption}
\theoremstyle{remark}

\usepackage{newfloat}
\usepackage{listings}
\DeclareCaptionStyle{ruled}{labelfont=normalfont,labelsep=colon,strut=off} 
\floatstyle{ruled}
\newfloat{listing}{tb}{lst}{}
\floatname{listing}{Listing}
\title{Addressing Polarization And Unfairness In Performative Prediction}

\author{
    Kun Jin\equalcontrib\textsuperscript{\rm 1},
    Tian Xie\equalcontrib\textsuperscript{\rm 2},
    Yang Liu\textsuperscript{\rm 3},
    Xueru Zhang\textsuperscript{\rm 2}
}
\affiliations{
    \textsuperscript{\rm 1} University of Michigan\\
    \textsuperscript{\rm 2} The Ohio State University\\
    \textsuperscript{\rm 3} University of California, Santa Cruz
}

\usepackage{bibentry}

\begin{document}

\maketitle

\begin{abstract}
In many real-world applications of machine learning—such as recommendations, hiring, and lending—deployed models influence the data they are trained on, leading to feedback loops between predictions and data distribution. The  \textit{performative prediction} (PP) framework captures this phenomenon by modeling the data distribution as a function of the deployed model. While prior work has focused on finding \textit{performative stable} (PS) solutions for robustness, their societal impacts, particularly regarding fairness, remain underexplored. We show that PS solutions can lead to severe polarization and prediction performance disparities, and that conventional fairness interventions in previous works often fail under model-dependent distribution shifts due to failing the PS criteria. To address these challenges in PP, we introduce novel fairness mechanisms that provably ensure \textbf{both stability and fairness}, validated by theoretical analysis and empirical results\footnote{\textbf{Code:} https://github.com/osu-srml/FairPP.}

\end{abstract}

\section{Introduction}\label{sec:intro}

Modern supervised learning has achieved remarkable success in static environments, where the data distribution remains unaffected by model deployment. However, in many real-world applications—such as digital platforms, hiring, or lending—models influence user behavior, causing feedback loops that shift the data distribution. These model-induced shifts often render standard training methods unstable or unfair, and they are prevalent in real-world applications. Examples include strategic individuals manipulating their data (in school admission, hiring, or lending) to game the ML system into making favorable predictions \citep{hardt_strategic_2015}, consumers changing their retention and participation choices (in digital platforms) based on their perception toward the ML model they are subject to \citep{Zhang_2019_Retention,chi2022towards}.  

To make predictions in the presence of model-induced distribution shifts, \citet{perdomo_performative_2021} proposed {\textbf{performative prediction}} (PP), a framework that explicitly considers the target data distribution $\mathcal{D}(\boldsymbol{\theta})$ as a function of the ML model parameter $\boldsymbol{\theta} \in \Theta \subset \mathbb{R}^d$ to be optimized in a compact domain. While PP captures the impact of ML model on target data, the distribution $\mathcal{D}(\boldsymbol{\theta})$ is solely determined by the model regardless of the original data distribution. A subsequent study \citep{brown2020performative} extended PP and proposed a more generalized  {\textbf{state-dependent performative prediction}} (SDPP) framework, which considers the impacts of both model and initial data distribution. Specifically, given the deployed ML model parameter $\boldsymbol{\theta}$ and initial data distribution $\mathcal{D}$, SDPP models the resulting target data distribution $\mathcal{D}' = \text{T}(\boldsymbol{\theta};\mathcal{D})$ using some transition mapping function $\text{T}$. Since PP is a special case of SDPP, we focus on SDPP in this paper.

When the transition map $\text{T}$ is $1$-jointly sensitive (details are in Def.~\ref{def:sens}), $\text{T}(\boldsymbol{\theta};\cdot)$ is contractive and repeatedly deploying $\boldsymbol{\theta}$ will cause the induced distributions to converge to a fixed point distribution $\mathcal{D}_{\boldsymbol{\theta}}$. The learning objective of SDPP is to minimize \textbf{\text{performative risk}} (PR) evaluated on $\mathcal{D}_{\boldsymbol{\theta}}$, i.e.,
\small
\begin{align*}
   \boldsymbol{\theta}^{\text{PO}} =\underset{\boldsymbol{\theta}}{\operatorname{argmin}}~\text{PR}(\boldsymbol{\theta}) \stackrel{\text{def}}{=} \mathbb{E}_{Z \sim \mathcal{D}_{\boldsymbol{\theta}}} [\ell(\boldsymbol{\theta};Z)] ~~~\text{s.t.}~~\mathcal{D}_{\boldsymbol{\theta}} = \text{T}(\boldsymbol{\theta};\mathcal{D}_{\boldsymbol{\theta}})
\end{align*}
\normalsize
\color{black}
where $\ell(\boldsymbol{\theta};Z)$ is the loss function, $Z$ is the data sampled from the \textit{fixed point} distribution $\mathcal{D}_{\boldsymbol{\theta}}$. 
The minimizer $\boldsymbol{\theta}^{\text{PO}}$ is named as \textbf{performative optimal} (PO) solution. Because the target data distribution itself is a function of variable $\boldsymbol{\theta}$ to be optimized, finding $\boldsymbol{\theta}^{\text{PO}}$ is often challenging \citep{perdomo_performative_2021, brown2020performative}.
Instead, existing works have mostly focused on finding  \textbf{performative stable} (PS) solution $\boldsymbol{\theta}^{\text{PS}}$, which minimizes the \textbf{decoupled performative risk} $\text{DPR}(\boldsymbol{\theta};\boldsymbol{\theta}^{\text{PS}})$ defined as follows,
\begin{equation}\label{eq:PS}
\boldsymbol{\theta}^{\text{PS}} = \underset{\boldsymbol{\theta}}{\operatorname{argmin}}~\text{DPR}(\boldsymbol{\theta};\boldsymbol{\theta}^{\text{PS}}) \stackrel{\text{def}}{=} \mathbb{E}_{Z \sim \text{T}(\boldsymbol{\theta}^{\text{PS}};\mathcal{D}^{\text{PS}})} [\ell(\boldsymbol{\theta};Z)]
\end{equation}
\normalsize
where $\mathcal{D}^{\text{PS}}$ is the fixed point data distribution induced by $\boldsymbol{\theta}^{\text{PS}}$ that satisfies $\mathcal{D}^{\text{PS}} = \text{T}(\boldsymbol{\theta}^{\text{PS}};\mathcal{D}^{\text{PS}})$. Unlike $\text{PR}(\boldsymbol{\theta})$ where data distribution  $\text{T}(\boldsymbol{\theta};\mathcal{D})$ depends on variable $\boldsymbol{\theta}$ to be optimized, $\text{DPR}(\boldsymbol{\theta};\boldsymbol{\theta}^{\text{PS}})$  decouples the two, i.e., data distribution is
induced by $\boldsymbol{\theta}^{\text{PS}}$ while the variable to be optimized is $\boldsymbol{\theta}$.
Although in general $\boldsymbol{\theta}^{\text{PS}}\neq \boldsymbol{\theta}^{\text{PO}}$, $\boldsymbol{\theta}^{\text{PS}}$ is the fixed point of \eqref{eq:PS} and stabilizes the system: at $\boldsymbol{\theta}^{\text{PS}}$, data distribution $\mathcal{D}^{\text{PS}}$ also remains fixed. Many algorithms have been proposed in the literature to find $\boldsymbol{\theta}^{\text{PS}}$. A prime example is \textit{repeated
risk minimization} (RRM) \citep{perdomo_performative_2021}, an iterative algorithm that finds $\boldsymbol{\theta}^{\text{PS}}$ (under certain conditions) by repeatedly updating the model $\boldsymbol{\theta}^{(t)}$ that minimizes risk on
the fixed distribution $\mathcal{D}^{(t-1)}$ induced by the previous model $\boldsymbol{\theta}^{(t-1)}$, i.e., 
\begin{eqnarray}\label{eq:rrm_o}
\boldsymbol{\theta}^{(t)} &=& \underset{\boldsymbol{\theta}}{\operatorname{argmin}}~ \mathbb{E}_{Z \sim \mathcal{D}^{(t-1)}} [\ell(\boldsymbol{\theta};Z)],\nonumber\\
 \mathcal{D}^{(t)} &=& \text{T}(\boldsymbol{\theta}^{(t)};\mathcal{D}^{(t-1)}).
\end{eqnarray}
    However, the societal impact of PS solutions is less understood and it is unclear whether PS solutions can cause harm and violate social norms such as fairness. 

In this paper, we examine the fairness properties of PS solutions. We consider scenarios where an ML model is used to make decisions about people from multiple social groups, and the population data distribution changes based on the ML model. We find that $\boldsymbol{\theta}^{\text{PS}}$ can 1) incur severe polarization effects: entire population $\mathcal{D}^{\text{PS}}$ is dominated by certain groups, leaving the rest marginalized and almost diminished in the system; 2) be biased when deployed on $\mathcal{D}^{\text{PS}}$ and people from different groups will experience different losses. Because in many important domains such as job or loan applications, it is critical to ensure equal quality of ML predictions and population diversity, we investigate under what conditions and by what algorithms we can simultaneously achieve stability and fairness in SDPP.  

Focusing on group-wise \textit{loss disparity} \citep{martinez2020minimax, diana2021minimax,khalili2023loss} and \textit{participation disparity} \citep{Zhang_2019_Retention, Raab_Boczar_Fazel_Liu_2024} fairness measures, we first explore whether existing fairness mechanisms commonly used in supervised learning can help mitigate unfairness in SDPP settings; this includes \textit{regularization methods} (adding fairness violation as a penalty term to the objective function of unconstrained optimization, e.g., \citep{khan2023fairness, zhang2021unified}) and \textit{re-weighting methods} (adjusting weights and importance of samples in learning objective, e.g., \citep{jung2023reweighting, duchi2018learning, duchi2023distributionally}). We show that common choices of penalty terms (e.g., group-wise loss difference) and re-weighting designs (e.g., standard distributionally robust optimization) that are effective in traditional supervised learning may fail in SDPP by disrupting the stability of the system. Using repeated risk minimization (RRM) shown in \eqref{eq:rrm_o}  as an example, this means that applying such fairness mechanisms at each round of RRM can disrupt the convergence of the iterative algorithm and $(\boldsymbol{\theta}^{(t)},\mathcal{D}^{(t)})$ may diverge to an unexpected state. We thus propose novel fairness mechanisms, which can be easily adopted and incorporated into iterative algorithms such as RRM. We theoretically show that the proposed mechanism can effectively improve fairness while maintaining the stability of the system.

It is worth noting that although a few recent works also studied fairness issues under model-induced distribution shifts, they all make rather strong assumptions about the distribution shifts and do not apply to the general SDPP framework. For example, \citet{mishler2022fair} pointed out the fairness issues under performative settings without providing solutions to achieve fairness and stability at the same time. \citet{zezulka2023performativity, hu2022achieving, Raab_Boczar_Fazel_Liu_2024, somerstep2024algorithmic} assumed there exists a causal model that depicts how data distribution would shift based on the ML model, and these causal models need to be fully known for the fairness mechanisms to work. \citet{Raab_Boczar_Fazel_Liu_2024} studied a special type of model-induced distribution shift where only the group proportion changes.  
In App. \ref{app:related}, we discuss more related works. 

The rest of the paper is organized as follows. Section \ref{section:formulation} provides the background of SDPP. Section \ref{sec:fair_issue} formulates the problem and demonstrates the unfairness and polarization issues of PS solutions. Section \ref{sec:fair_mechanism} highlights the difficulties of simultaneously achieving fairness and stability in SDPP, where we first show that existing fairness mechanisms commonly used in supervised learning may fail in SDPP settings and then propose a novel fairness mechanism. 
In Section \ref{section:convergence}, we conduct the theoretical analysis and show that our method can effectively improve fairness while maintaining stability. Finally, Section \ref{section:numerical} empirically validates the proposed method on both synthetic and real data.

\section{Preliminaries} \label{section:formulation}

\paragraph{Iterative algorithms to find PS solutions.} As mentioned in Section~\ref{sec:intro}, the original goal of SDPP is to find $\boldsymbol{\theta}^{\text{PO}}$ that minimizes $\text{PR}(\boldsymbol{\theta})=\mathbb{E}_{Z \sim \mathcal{D}_{\boldsymbol{\theta}}} [\ell(\boldsymbol{\theta};Z)]$, the loss over the population induced by the deployed model. However, solving this optimization is often challenging because the data distribution $\mathcal{D}_{\boldsymbol{\theta}}$ depends on the variable $\boldsymbol{\theta}$ being optimized. Thus, prior studies such as \cite{perdomo_performative_2021, brown2020performative} have mostly focused on finding performative stable solution $\boldsymbol{\theta}^{\text{PS}}$, which is the fixed point of Eqn.~\eqref{eq:PS} and can be found through an iterative process of \textit{data sampling} and \textit{model deployment}. Specifically, denote $\mathcal{L}(\boldsymbol{\theta};\mathcal{D})=\mathbb{E}_{Z \sim \mathcal{D}} [\ell(\boldsymbol{\theta};Z)]$ and let $(\boldsymbol{\theta}^{(t)},\mathcal{D}^{(t)})$ be the model parameter and data distribution at round $t$ of the iterative algorithm, then repeatedly updating the model $\boldsymbol{\theta}^{(t)}$  according to Eqn.~\eqref{eq:rrm_original} could lead $(\boldsymbol{\theta}^{(t)},\mathcal{D}^{(t)})$ converging to PS solution  $(\boldsymbol{\theta}^{\text{PS}},\mathcal{D}^{\text{PS}})$ under certain conditions \citep{perdomo_performative_2021, brown2020performative}. 
\begin{eqnarray}\label{eq:rrm_original}
\boldsymbol{\theta}^{(t)} =\underset{\boldsymbol{\theta}}{\operatorname{argmin}}~ \mathcal{L}(\boldsymbol{\theta};\mathcal{D}^{(t-1)}),~~
\mathcal{D}^{(t)} =\text{Tr}(\boldsymbol{\theta}^{(t)};\mathcal{D}^{(t-1)}).
\end{eqnarray}
where $\text{Tr}$ may not be the same as transition map $\text{T}$ that drives evolution of data. Depending on how frequently the model is deployed compared to the change of data, $\text{Tr}$ is defined differently based on \textbf{repeated deployment schema}. Common examples include:
\begin{align*}
\textbf{conventional:~} & \text{Tr}(\boldsymbol{\theta}^{(t)};\mathcal{D}^{(t-1)}) = \text{T}(\boldsymbol{\theta}^{(t)};\mathcal{D}^{(t-1)}) \\
\textbf{$\mathbf{k}$-delayed:~} & \text{Tr}(\boldsymbol{\theta}^{(t)};\mathcal{D}^{(t-1)}) = \text{T}^k(\boldsymbol{\theta}^{(t)};\mathcal{D}^{(t-1)}) \\
= &\underbrace{\text{T}\Big(\boldsymbol{\theta}^{(t)}; \dots \text{T}\big( \boldsymbol{\theta}^{(t)};\text{T}}_{k \text{ times}}(\boldsymbol{\theta}^{(t)};\mathcal{D}^{(t-1)})\big)\Big)\\
\textbf{delayed:~} & \text{Tr}(\boldsymbol{\theta}^{(t)};\mathcal{D}^{(t-1)}) = 
\text{T}^{\lceil r \rceil + 1}(\boldsymbol{\theta}^{(t)};\mathcal{D}^{(t-1)})
\end{align*}
where the \textit{repeated risk minimization} (RRM) \citep{perdomo_performative_2021} introduced in Section~\ref{sec:intro} corresponds to conventional deployment schema. By customizing the time interval between two deployments, we can get variants including \textit{delayed RRM} and \textit{$k$-delayed RRM} \citep{brown2020performative}. Note that the delayed deployment schema is a special case of $k$-delayed deployment schema, where the number of repeated deployments $r$ is chosen to ensure the output distribution $\mathcal{D}^{(t)}$ is sufficiently close to the fixed point distribution when $\boldsymbol{\theta}^{(t)}$ keeps being deployed on the population $\mathcal{D}^{(t-1)}$.

\paragraph{Technical conditions for iterative algorithms to converge to PS solutions.}
As shown in \cite{perdomo_performative_2021,brown2020performative}, PS solutions exist and are unique only when $\ell$ and $\text{T}$ satisfy certain conditions. Moreover, iterative algorithms introduced in Eqn. \eqref{eq:rrm_original} can converge to the PS solution. We introduce these conditions below, where $\Theta$, $\mathcal{Z}$, and $\triangle (\mathcal{Z})$ denote the parameter space, sample space, and space of distributions over samples.
 
\begin{definition} [Strong convexity of loss function]
    $\ell (\boldsymbol{\theta};Z)$ is $\gamma$-strongly convex if and only if for all $\boldsymbol{\theta}, \boldsymbol{\theta}' \in \Theta$ and $Z \in \mathcal{Z}$, we have
    \begin{equation*}
        \ell (\boldsymbol{\theta};Z) \geq \ell (\boldsymbol{\theta}';Z) + \langle \nabla_{\boldsymbol{\theta}} \ell (\boldsymbol{\theta'};Z),  \boldsymbol{\theta} - \boldsymbol{\theta}' \rangle + \frac{\gamma}{2} \|\boldsymbol{\theta} -  \boldsymbol{\theta}' \|_2^2.
    \end{equation*}
\end{definition}

\begin{definition} [Joint smoothness of loss function] \label{def:joint_smootheness}
    $\ell (\boldsymbol{\theta}; Z)$ is $\beta$-jointly smooth if the gradient with respect to $\boldsymbol{\theta}$ is $\beta$-Lipschitz in $\boldsymbol{\theta}$ and $Z$, i.e., $\forall \boldsymbol{\theta}, \boldsymbol{\theta}' \in \Theta$ and $\forall  Z, Z' \in \mathcal{Z}$, we have
    \begin{align*}
        \|\nabla_{\boldsymbol{\theta}} \ell (\boldsymbol{\theta}; Z) - \nabla_{\boldsymbol{\theta}} \ell (\boldsymbol{\theta}'; Z)\|_2 &\leq \beta \| \boldsymbol{\theta} - \boldsymbol{\theta}'\|_2\\        \|\nabla_{\boldsymbol{\theta}} \ell (\boldsymbol{\theta}; Z) - \nabla_{\boldsymbol{\theta}} \ell (\boldsymbol{\theta}; Z')\|_2 &\leq \beta  \| Z - Z'\|_2
    \end{align*}
\end{definition}

\begin{definition} [Joint sensitivity of transition map]\label{def:sens}
    Let $\mathcal{W}_1$ denote the Wasserstein-1 distance measure. The transition map $\text{T}$ is $\epsilon$-jointly sensitive if for all $\boldsymbol{\theta}, \boldsymbol{\theta}' \in \Theta$ and $\mathcal{D}, \mathcal{D}' \in \triangle (\mathcal{Z})$, we have
    \begin{align*}
    \mathcal{W}_1( \text{T}(\boldsymbol{\theta}; \mathcal{D}), \text{T}(\boldsymbol{\theta}'; \mathcal{D})) &\leq \epsilon \|\boldsymbol{\theta} -\boldsymbol{\theta}' \|_2 \\
    \mathcal{W}_1( \text{T}(\boldsymbol{\theta}; \mathcal{D}), \text{T}(\boldsymbol{\theta}; \mathcal{D}')) &\leq \epsilon \mathcal{W}_1(\mathcal{D},\mathcal{D}')
    \end{align*}
\end{definition}

\begin{lemma}[Existence of a unique PS solution \citep{brown2020performative, perdomo_performative_2021}] \label{lemma:SDPP}
SDPP problem is guaranteed to have a unique PS solution if \textbf{all} of the following hold: (i) $\ell (\boldsymbol{\theta};Z)$ is $\gamma$-strongly convex; (ii) $\ell (\boldsymbol{\theta};Z)$ is $\beta$-joint smooth; (iii) $\text{T}$ is $\epsilon$-joint sensitive and $\epsilon(1+2\beta/\gamma) < 1$. 
\end{lemma}
\begin{lemma}[Convergence of iterative algorithms]\label{lemma:converge}
If conditions (i)(ii)(iii) in Lemma \ref{lemma:SDPP} are all satisfied, iterative algorithms are guaranteed to converge to the unique solution. However, there is no convergence guarantee when $\ell (\boldsymbol{\theta}; Z)$ is non-convex even if both (ii) and (iii) are satisfied. 
\end{lemma}

Lemma \ref{lemma:SDPP} has been shown in Theorem 8 of \citet{brown2020performative}, while we prove Lemma \ref{lemma:converge}
based on \citet{perdomo_performative_2021} in App. \ref{app:proof_SDPP}.

\section{Unfairness and polarization in SDPP}\label{sec:fair_issue}

\paragraph{Problem formulation.} In this work, we study SDPP with different demographic groups, where an ML model $\boldsymbol{\theta}$ is trained to make predictions about individuals from multiple groups distinguished by a sensitive attribute $s \in \mathcal{S}$ (e.g., gender, age, race), whose data distribution changes based on the deployed ML model and such model-induced distribution shift can be captured by transition map $\text{T}$. Suppose individuals from group $s$ follow the identical data distribution $\mathcal{D}_s^{(t)}$ at the round $t$ of an iterative algorithm, and let $p_s^{(t)}$ be the size of group $s$ as the fraction of entire population at $t$. Then the data distribution of the entire population is  $\mathcal{D}^{(t)} = \sum_{s \in \mathcal{S}} p_s^{(t)} \mathcal{D}_s^{(t)}$ with $\sum_{s \in \mathcal{S}} p_s^{(t)} =1$. 

Note that the above SDPP with multiple groups 
is a general framework. By specifying the transition mapping $\text{T}$, many problems studied in prior works can be regarded as a special case. This includes:
\begin{enumerate}[leftmargin=*,topsep=0.2cm,itemsep=0cm]
    \item \textbf{Strategic classification} \citep{hardt_strategic_2015, Zhang_2022_ICML}: individuals in high-stakes applications such as lending, hiring, and college admission may manipulate their data based on ML model strategically to increase their chances of receiving favorable decisions, leading to changes in group distribution $\mathcal{D}_s^{(t)}$.   
    \item \textbf{Decision-making systems under user retention dynamics} \citep{Zhang_2019_Retention, duchi2018learning,pmlr-v80-hashimoto18a}: ML models in recognition or recommendation systems may attract more users if they experience high accuracy but drive away those with less satisfaction, causing the group proportion $p_s^{(t)}$ to change.  
\end{enumerate}

We first explore the fairness properties of PS solutions in SDPP, i.e., examining whether $\boldsymbol{\theta}^{\text{PS}}$ in Eqn.~\eqref{eq:PS} have disparate impacts on different demographic groups. Specifically, we consider two fairness metrics: group-wise \textit{loss disparity} $\triangle_{\mathcal{L}}^{(t)}$ \citep{martinez2020minimax, hashimoto2018fairness} and \textit{participation disparity} $\triangle_p^{(t)}$ \citep{Raab_Boczar_Fazel_Liu_2024}, which measure the difference of group loss $\mathcal{L}(\boldsymbol{\theta}; \mathcal{D}_s^{(t)})$ and fraction $p_s^{(t)}$ across different groups at round $t$ of an iterative algorithm, respectively. In examples with two groups $\mathcal{S}=\{a,b\}$, the unfairness can be quantified as:
\begin{align}\label{eq:unfairness}
 \triangle_{\mathcal{L}}^{(t)} &:= \left|\mathcal{L}(\boldsymbol{\theta}^{(t)}; \mathcal{D}_a^{(t)}) - \mathcal{L}(\boldsymbol{\theta}^{(t)}; \mathcal{D}_b^{(t)})\right|, \nonumber \\\triangle_p^{(t)}&:=\left|p_a^{(t)}-p_b^{(t)}\right| 
\end{align}

\paragraph{Unfairness \& polarization effects in SDPP.}

\begin{figure}
\includegraphics[trim=0.21cm 0.35cm 0.25cm 0.2cm,clip, width=0.21\textwidth]{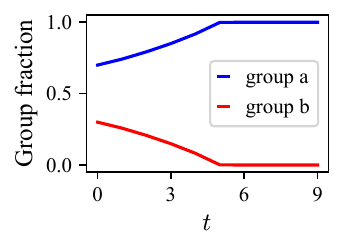}
\includegraphics[trim=0.21cm 0.35cm 0.25cm 0.2cm,clip, width=0.21\textwidth]{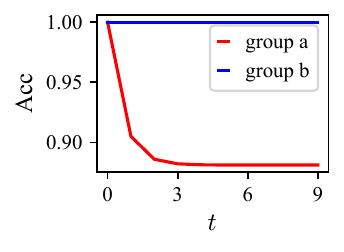}
\centering
\vspace{-0.1cm}
\caption{Illustrating examples of polarization effects and unfairness of $\boldsymbol{\theta}^{\text{PS}}$ in Prop.~\ref{prop1} (left) and \ref{prop2} (right): dynamics of the group fraction (left) and group-wise accuracy (right) under RRM: the system converges to  $(\boldsymbol{\theta}^{\text{PS}}; \mathcal{D}^{\text{PS}})$ that is unfair (details are in App. \ref{example:ps}).  }
\label{fig:e34}
\end{figure}

We first show that PS solutions $\boldsymbol{\theta}^{\text{PS}}$ in SDPP may have disparate impacts on different groups. Specifically, when finding $\boldsymbol{\theta}^{\text{PS}}$ using iterative algorithms introduced in Section~\ref{section:formulation}, the process may incur severe \textit{polarization effects} and exhibit \textit{unfairness}, i.e., certain groups may get more and more marginalized, and group-wise loss disparity gets exacerbated during the iterative process.

\begin{proposition}[Polarization effects of $\boldsymbol{\theta}^{\text{PS}}$]\label{prop1} Consider  population from multiple groups with fixed group distribution $\mathcal{D}_s$ whose participation $p_s^{(t+1)}$ in an ML system depends on their perceived group loss $\mathcal{L}(\boldsymbol{\theta}^{(t)}; \mathcal{D}_s)$. Suppose the deployment of system $\boldsymbol{\theta}^{(t)}=\arg\min_{\boldsymbol{\theta}} \sum_{s\in\mathcal{S}}p_s^{(t)}\mathcal{L}(\boldsymbol{\theta}; \mathcal{D}_s)$ follows the conventional RRM schema, then   
there exist  $\mathcal{D}^{(0)}$ and $\text{T}$ such that as $t\to \infty$, $p_s^{(t)}$ changes monotonically and certain groups diminish entirely from the system. 
\end{proposition}

\begin{proposition}[Exacerbated group-wise loss disparity]\label{prop2}
 Consider population from multiple groups with fixed group proportion $p_s$; each individual is subject to a binary ML decision $\widehat{Y}^{(t)}=\mathbf{1}(X^{(t)}\ge \boldsymbol{\theta}^{(t)})$ and may strategically manipulate the data to increase the chance of receiving positive decisions. Suppose the deployment of ML system $\boldsymbol{\theta}^{(t)}=\arg\max_{\boldsymbol{\theta}} \Pr\big(\widehat{Y}^{(t)}= Y\big)$ follows the conventional RRM schema, and individuals manipulate features according to $X^{(t)} = X^{(t)} + \eta \boldsymbol{\theta}^{(t)}$ without changing $Y$\citep{perdomo_performative_2021}, then there exists $\mathcal{D}^{(0)}$ such that group-wise loss disparity $\Delta^{(t)}_{\mathcal{L}}$ increases. 
\end{proposition}

To prove Prop. \ref{prop1} and \ref{prop2}, it is sufficient to provide two examples to illustrate the polarization and unfairness effects in SDPP settings. We construct the examples with details in App. \ref{example:ps} and visualize them in Figure \ref{fig:e34}. This shows that even though RRM can converge to a stable solution $\boldsymbol{\theta}^{\text{PS}}$, the solution is unfair and loss disparity $\triangle_{\mathcal{L}}^{(t)}$ and participation disparity $\triangle_{p}^{(t)}$ may get exacerbated.
\color{black}

\section{Finding Fair-PS solutions} \label{sec:fair_mechanism}
Section~\ref{sec:fair_issue} shows that without fairness consideration, PS solutions of SDPP may incur polarization effects and have disparate impacts on different groups. This section tackles unfairness issues in SDPP. One straightforward idea is to directly apply the fairness mechanisms at every round of the iterative algorithms introduced in Section~\ref{section:formulation}. However, we will show that although such methods are effective in conventional supervised learning, they can disrupt the stability and the iterative algorithms may no longer converge.


\subsection{Fair-PS solutions}

Many fairness mechanisms have been proposed in supervised learning to mitigate group-wise loss disparity and participation disparity. We consider two categories commonly used in the literature: \textit{regularization method} and \textit{sample re-weighting method}, as detailed below. 
\begin{enumerate}[leftmargin=*]
    \item \textbf{Fairness via regularization:} It adds a regularization or penalty term to the original learning objective function $\mathcal{L}(\boldsymbol{\theta}; \mathcal{D})$, which penalizes the violation of fairness \citep{khan2023fairness, zhang2021unified}. The fair objective function is
\begin{equation}\label{eq:fair_reg_risk}
\mathcal{L}_{\textsf{fair}}(\boldsymbol{\theta}; \mathcal{D}, \rho) := \mathcal{L}(\boldsymbol{\theta}; \mathcal{D}) + \mathcal{P}(\boldsymbol{\theta}; \mathcal{D},\rho),
\end{equation}
where $\mathcal{P}(\boldsymbol{\theta}; \mathcal{D},\rho)$ is the fair penalty term and the scalar $\rho > 0$ controls the strength of the penalty.

\item \textbf{Fairness via sample re-weighting:} It adjusts the weights of samples (possibly adversarially) and increases weights for disadvantaged groups \citep{jung2023reweighting, duchi2018learning, duchi2023distributionally}. An example is distributionally robust optimization (DRO) \citep{pmlr-v80-hashimoto18a}, which minimizes worst-case  loss and the fair objective is
\begin{equation} \label{eq:DRO_objective}
 \textstyle   \mathcal{L}_{\textsf{fair}}(\boldsymbol{\theta}; \mathcal{D}, \rho) :=  \max_{\Tilde{\mathcal{D}} \in \mathcal{B}(\mathcal{D}, r(\rho))} \mathcal{L}(\boldsymbol{\theta}; \Tilde{\mathcal{D}}), 
\end{equation}
where $\mathcal{B}(\mathcal{D}, r(\rho)) := \{\Tilde{\mathcal{D}} | d(\mathcal{D}, \Tilde{\mathcal{D}}) \leq r(\rho)\}$ denotes a distribution ball centered at $\mathcal{D}$ with radius $r(\rho)$ derived from the fair mechanism strength $\rho$, and $d$ is a distribution distance metric. 
\end{enumerate}
By optimizing a fair objective $ \mathcal{L}_{\textsf{fair}}(\boldsymbol{\theta}; \mathcal{D}, \rho)$, existing fairness mechanisms can effectively mitigate unfairness in supervised learning with static data distribution $\mathcal{D}$. However, it remains unclear how these methods would perform in SDPP when the model itself causes the data distribution shifts. Specifically, consider iterative algorithms introduced in Eqn. \eqref{eq:rrm_original} that find PS solutions (e.g., RRM). Suppose we apply the above fairness mechanism at every round when updating the model parameter $\boldsymbol{\theta}$, i.e., replacing $\boldsymbol{\theta}^{(t)} = \operatorname{argmin}_{\boldsymbol{\theta}} \mathcal{L}(\boldsymbol{\theta}; \mathcal{D}^{(t-1)})$ with fair version $\boldsymbol{\theta}^{(t)} = \operatorname{argmin}_{\boldsymbol{\theta}} \mathcal{L}_{\textsf{fair}}(\boldsymbol{\theta}; \mathcal{D}^{(t-1)}, \rho)$ in iterative algorithms. We ask: \textit{can such new iterative algorithms mitigate group-wise loss and participation disparity in SDPP and converge to a fair and stable solution?} 

Before answering the above question, we first define \textbf{Fair-PS solutions} for SDPP, at which both ML system and population distribution reach stability and unfairness is mitigated.

\begin{definition}[Fair-PS solution]\label{def:fair-ps}
    We define $(\mathcal{D}^{\text{PS}}_{\textsf{fair}}, \boldsymbol{\theta}^{\text{PS}}_{\textsf{fair}})$ as the Fair-PS solution to $ \mathcal{L}_{\textsf{fair}}(\boldsymbol{\theta}; \mathcal{D}, \rho)$ if
    $$ \mathcal{D}^{\text{PS}}_{\textsf{fair}} = \text{T}(\boldsymbol{\theta}^{\text{PS}}_{\textsf{fair}};\mathcal{D}^{\text{PS}}_{\textsf{fair}}),~~\boldsymbol{\theta}^{\text{PS}}_{\textsf{fair}} = \underset{\boldsymbol{\theta}}{\operatorname{argmin}}~ \mathcal{L}_{\textsf{fair}}(\boldsymbol{\theta}; \mathcal{D}^{\text{PS}}_{\textsf{fair}},\rho).$$
\end{definition}

\subsection{Existing designs fail to converge to Fair-PS solutions}
We will use two examples to illustrate that the popular choices of fairness mechanisms used in existing  literature may fail to achieve stability and fairness in SDPP. This includes (i) regularization method \citep{khan2023fairness, zhang2021unified} with \textit{group loss variance} $
        \mathcal{P}(\boldsymbol{\theta}; \mathcal{D},\rho) := \rho \sum_{s\in\mathcal{S}} p_s \cdot [\mathcal{L}(\boldsymbol{\theta}; \mathcal{D}_s) - \mathcal{L}(\boldsymbol{\theta}; \mathcal{D})]^2$  as the penalty term in Eqn.~\eqref{eq:fair_reg_risk}; and (ii) \textit{distributionally robust optimization} (DRO) method  \citep{hashimoto2018fairness,peetpare2023long} with $\chi^2$-distance metric
        $d(\mathcal{D},\Tilde{\mathcal{D}}) := \int \left(\frac{d \mathcal{D}}{d \Tilde{\mathcal{D}}} - 1\right)^2 d \Tilde{\mathcal{D}}$ in Eqn.~\eqref{eq:DRO_objective}. 

\begin{example}[Group loss variance as penalty term] \label{example:group_loss_var} 
Consider two groups $a, b$ with data $Z = (X,Y)$ and fixed group fractions $p_a = p_b = 0.5$. Suppose all samples in group $a$ are $(1, -1)$ and all samples in group $b$ are $(2,1)$. Consider squared loss function $\ell(z; \theta) = (y - h_{\theta}(x))^2 = (y - \theta x)^2$ which is strongly convex and jointly smooth. Under group loss variance penalty $\mathcal{P}(\boldsymbol{\theta}; \mathcal{D},\rho)$, we have 
$ \mathcal{L}_{\textsf{fair}}(\boldsymbol{\theta}; \mathcal{D}, \rho) = 2.5 \theta^2 -\theta + 1 + \rho \cdot (2.25 \theta^4 + 9\theta^2 - 9\theta^3)$. When $\rho = 0.6$, the second-order derivative $\nabla^2_{\theta} \mathcal{L}_{\textsf{fair}} = 16.2 \theta^2 -32.4 \theta +15.8$ and is negative when $\theta = 1$. The negative second-order gradient means that $\mathcal{L}_{\textsf{fair}}$ is nonconvex. According to Lemma \ref{lemma:SDPP}, we cannot \textbf{ensure} the iterative algorithms converge to Fair-PS solutions when $\ell$ is nonconvex. Thus, adding group loss variance as a penalty at each round possibly disrupts the stability. Appendix \ref{app:e5} verifies the non-convergence with empirical results.
\end{example}

\begin{example}[Repeated DRO with $\chi^2$-distance metric]\label{example:dro}
Consider two groups $a,b$ with fixed data $z_a =1$ and $z_b=-1$ for all samples but group fractions $p_a^{(t)} = 0.5\cdot(1+\theta^{(t)})$ and $p_b^{(t)} = 0.5\cdot(1-\theta^{(t)})$. $p_a^{(0)} = 0.4, p_b^{(0)} = 0.6$. Consider a mean estimation task where the model parameter ${\theta}^{(t+1)} := \arg\min_{\boldsymbol{\theta}} \max_{\Tilde{\mathcal{D}} \in \mathcal{B}(\mathcal{D}^{(t)}, r)} \mathcal{L}({\theta}; \Tilde{\mathcal{D}})$ is updated using DRO with mean squared error and $\chi^2$-distance bound $r = 1/6$.  Denote $q_a^{(t)},q_b^{(t)}$as the group fractions of the "worst-case" distribution $\Tilde{\mathcal{D}}$. We  have $q_a^{(0)} = 0.6$, $q_b^{(0)} = 0.4$ and $\theta^{(1)} = \arg\min_{\theta} q_a^{(0)}(1-\theta)^2 + q_b^{(0)}(1+\theta)^2=0.2$, which results in $p_a^{(1)} = 0.6, p_b^{(1)} = 0.4$. Since $\mathcal{L}_a(\theta^{(1)};z_a) < \mathcal{L}_b(\theta^{(1)};z_b)$, DRO should minimize the risk of the "worst-case" distribution with $q_a^{(1)} = 0.4$, $q_b^{(1)} = 0.6$, i.e., $\theta^{(2)} = \arg\min_{\theta} q_a^{(1)}(1-\theta)^2 + q_b^{(1)}(1+\theta)^2 = -0.2$. Repeating the procedure we will get $p_a^{(2)} = 0.4, p_b^{(2)} = 0.6$ and  $q_a^{(2)} = 0.6, q_b^{(2)} = 0.4$, $\theta^{(3)} = 0.2$. It turns out that repeated DRO results in $\theta^{(t)}$ oscillating between $0.2$ and $-0.2$ and it never converges.
\end{example}

It is worth noting that DRO methods have been used in \citet{hashimoto2018fairness} to mitigate group fraction disparity in repeated optimizations. However, it only improves fairness without any convergence guarantees to stable solutions. \citet{peetpare2023long} repeatedly used DRO to improve fairness under PP settings, it only converges to a PS solution under stronger assumptions where the distributionally robust objective must be strongly convex and jointly smooth, and the transition map $\text{T}$ also needs to be $\epsilon$-sensitive with respect to the worst-case distribution. Under milder conditions in Lemma \ref{lemma:SDPP}, it may fail to converge as Example \ref{example:dro} illustrated.
\color{black}

\subsection{Novel designs for fairness mechanism}\label{subsec:fair}
Next, we introduce three novel fair objective functions $ \mathcal{L}_{\textsf{fair}}(\boldsymbol{\theta}; \mathcal{D}, \rho)$ for fairness mechanisms, of which two belong to \textit{regularization method} and one is a \textit{sample re-weighting method}. By replacing $\boldsymbol{\theta}^{(t)} = \operatorname{argmin}_{\boldsymbol{\theta}} \mathcal{L}(\boldsymbol{\theta}; \mathcal{D}^{(t-1)})$ with the proposed fair update $\boldsymbol{\theta}^{(t)} = \operatorname{argmin}_{\boldsymbol{\theta}} \mathcal{L}_{\textsf{fair}}(\boldsymbol{\theta}; \mathcal{D}^{(t-1)}, \rho)$ in Eqn. \eqref{eq:rrm_original}, the resulting iterative algorithms can converge to Fair-PS solutions.


\paragraph{Proposed fair regularization (with and without demographics).}
Let $\mathcal{D}^{(0)}$ denote the initial population distribution. Depending on whether sensitive attributes $S$ are accessible during training, we propose two fairness penalty terms, as detailed below. 
\begin{enumerate}[leftmargin=*,topsep=0.2cm,itemsep=0.0cm]
    \item \underline{{Group level fairness penalty:}} It updates  $\boldsymbol{\theta}^{(t)}$ by minimizing $\mathcal{L}_{\textsf{fair}}(\boldsymbol{\theta}; \mathcal{D}^{(t-1)}, \rho)$ defined as follows
\small
\begin{align} \label{eq:squared_group_loss_penalty}
\mathcal{L}(\boldsymbol{\theta}; \mathcal{D}^{(t-1)})+  \rho \sum_{s\in\mathcal{S}} p^{(t-1)}_s  [\mathcal{L}(\boldsymbol{\theta}; \mathcal{D}_s^{(t-1)})]^2
\end{align}
\normalsize
    \item \underline{{Sample level fairness penalty without demographics:}} It updates $\boldsymbol{\theta}^{(t)}$ by minimizing $\mathcal{L}_{\textsf{fair}}(\boldsymbol{\theta}; \mathcal{D}^{(t-1)}, \rho)$ defined as follows, which does not require access to sensitive attribute values.
\begin{align} 
\label{eq:fair_loss_sample_level}
\textstyle \mathcal{L}(\boldsymbol{\theta}; \mathcal{D}^{(t-1)}) +\rho\mathbb{E}_{Z \sim \mathcal{D}^{(t-1)}} \big[  [\ell(\boldsymbol{\theta};Z)]^2 \big].
\end{align}
\end{enumerate}

\paragraph{Proposed fair sample re-weighting.} At the first round, it performs risk minimization. Starting from $t=2$, it updates $\boldsymbol{\theta}^{(t)}$ by minimizing $\mathcal{L}_{\textsf{fair}}(\boldsymbol{\theta}; \mathcal{D}^{(t-1)}, \rho)$ defined as follows
\begin{align}\label{eq:fair_reweight}
\textstyle \sum_{s\in\mathcal{S}} q_s^{(t-1)} \mathcal{L}(\boldsymbol{\theta}; \mathcal{D}_s^{(t-1)})
\end{align}
with
\begin{align*}
\textstyle \boldsymbol{q}^{(t-1)} &= \left[q^{(t-1)}_s\right]_{s \in \mathcal{S}} = \frac{\boldsymbol{p}^{(t-1)} + \rho \boldsymbol{l}^{(t-1)}} {\| \boldsymbol{p}^{(t-1)} + \rho \boldsymbol{l}^{(t-1)} \|_1} \\ \textstyle\boldsymbol{l}^{(t-1)} &= \left[p_s^{(t-1)} \mathcal{L}(\boldsymbol{\theta}^{(t-1)};\mathcal{D}_s^{(t-2)}) \right]_{s \in \mathcal{S}}
\end{align*}
Unlike DRO in \citet{peetpare2023long, hashimoto2018fairness} that requires solving a min-max optimization at the current round, our re-weighting method only adjusts the weights for each group based on the group-wise losses in the previous round, which is more computationally efficient.

\paragraph{Comparison \& discussion.}
Intuitively, compared to original $\mathcal{L}(\boldsymbol{\theta}; \mathcal{D}^{(t-1)})$, all three proposed fair objective functions $ \mathcal{L}_{\textsf{fair}}(\boldsymbol{\theta}; \mathcal{D}^{(t-1)},\rho)$ improves fairness at each round by assigning more weights to disadvantaged groups/samples (i.e., those experiencing higher losses) in the upcoming update. Indeed, both \textit{sample level fairness penalty} and \textit{fair sample re-weighting} can be regarded as modifications of \textit{group level fairness penalty}. Comparing Eqn.~\eqref{eq:squared_group_loss_penalty} and \eqref{eq:fair_loss_sample_level}, the two penalty terms get similar when most individual samples in the disadvantaged (resp. advantaged) groups are also similarly disadvantaged (resp. advantaged). This is more likely to happen when each group has a distribution with a small variance. For example, if the distribution of each group is a point mass, Eqn. \eqref{eq:squared_group_loss_penalty} and \eqref{eq:fair_loss_sample_level} are identical.

Comparing Eqn. \eqref{eq:squared_group_loss_penalty} and \eqref{eq:fair_reweight}, we can rewrite Eqn. \eqref{eq:fair_reweight} as the following:
\begin{align*}
 &~~~~~~~\| \boldsymbol{p}^{(t-1)} + \rho \boldsymbol{l}^{(t-1)}\|_1  \mathcal{L}_{\textsf{fair}} (\boldsymbol{\theta};\mathcal{D}^{(t-1)},\rho)\\
 &= \sum_{s} p_s^{(t-1)}(1+\rho \mathcal{L}(\boldsymbol{\theta}^{(t-1)}; \mathcal{D}_s^{(t-2)}))\,\mathcal{L}(\boldsymbol{\theta}; \mathcal{D}_s^{(t-1)})  
\end{align*}
We can see that the right-hand side will be a multiple of Eqn. \eqref{eq:squared_group_loss_penalty} if we replace $\mathcal{L}(\boldsymbol{\theta}^{(t-1)};\mathcal{D}_s^{(t-2)})$ with $\mathcal{L}(\boldsymbol{\theta};\mathcal{D}_s^{(t-1)})$. This means both equations yield the same $\boldsymbol{\theta}^{(t)}$ when the population distribution does not change from $t-2$ to $t-1$, suggesting that the two approaches become more similar when the sensitivity of the transition map $\text{T}$ is smaller, or equivalently, the distribution shift is milder.

\section{Theoretical analysis} \label{section:convergence}

\begin{algorithm}[H]
 {\small \caption{Fair repeated risk minimization (Fair-RRM) }
   \label{alg:fair_RRM}
\begin{algorithmic}
   \REQUIRE $t=0$, initial data distribution  $\mathcal{D}^{(0)}$, strength of fair mechanism $\rho$,  initial model parameter $\boldsymbol{\theta}^{(0)}$, stopping criteria $\tau$ 
   \STATE Choose repeated deployment schema and fair mechanism; 
   \REPEAT
   \STATE $\boldsymbol{\theta}^{(t+1)} \leftarrow \arg \min_{\boldsymbol{\theta}}  \mathcal{L}_{\textsf{fair}}(\boldsymbol{\theta};\mathcal{D}^{(t)},\rho)$;
   \STATE Get $\mathcal{D}^{(t+1)}=\text{Tr}(\boldsymbol{\theta}^{(t+1)};\mathcal{D}^{(t)}) $ from the chosen schema;
   \STATE $t \leftarrow t+1$;
   \UNTIL{$\| \boldsymbol{\theta}^{(t)} - \boldsymbol{\theta}^{(t-1)}\|_2 \leq \tau$}
\end{algorithmic}
}
\end{algorithm}

In this section, we will show that by replacing $\boldsymbol{\theta}^{(t)} = \operatorname{argmin}_{\boldsymbol{\theta}} \mathcal{L}(\boldsymbol{\theta}; \mathcal{D}^{(t-1)})$  in Eqn.~\eqref{eq:rrm_original} with the fair update $\boldsymbol{\theta}^{(t)} = \operatorname{argmin}_{\boldsymbol{\theta}} \mathcal{L}_{\textsf{fair}}(\boldsymbol{\theta}; \mathcal{D}^{(t-1)}, \rho)$ we proposed in Section~\ref{subsec:fair}, Fair-PS solutions (Definition~\ref{def:fair-ps}) exist under certain conditions and the resulting iterative algorithms (Algorithm \ref{alg:fair_RRM}) can converge to such Fair-PS solutions. We assume conditions in Lemma~\ref{lemma:SDPP} hold in this section and there exists a unique PS solution in the original SDPP problem. We first define a parameter $\Tilde{\beta}$ which will be frequently used in the theorems introduced below.
\small
$$\Tilde{\beta}:= \begin{cases}
     (2 \rho \overline{\ell} + 1)\beta + 2\rho\tilde{\ell}^2, ~\text{for regularization method \eqref{eq:squared_group_loss_penalty} or \eqref{eq:fair_loss_sample_level} }\\
     (\rho \overline{\ell} + 1)\beta,~\text{for sample re-weighting method \eqref{eq:fair_reweight} }
\end{cases}$$
\normalsize
where $\overline{\ell} := \sup_{\boldsymbol{\theta},Z} \|\ell(\boldsymbol{\theta};Z)\|$ and $\tilde{\ell} := \sup_{\boldsymbol{\theta},Z}\{ \|\nabla\ell_{\boldsymbol{\theta}}(\boldsymbol{\theta};Z)\|, \|\nabla\ell_{Z}(\boldsymbol{\theta};Z)\|\}$. We can identify conditions under which a unique Fair-PS solution exists. 

\begin{proposition}[Existence of unique Fair-PS solution] \label{prop:unique_fair_PS} 
  For a given population with initial distribution $\mathcal{D}^{(0)}$ and the proposed fair mechanism $\mathcal{L}_{\textsf{fair}}$ with strength $\rho$, there  is a unique Fair-PS solution if $\epsilon(1 + \Tilde{\beta}/\gamma) < 1$. Moreover,  $(\mathcal{D}^{\text{PS}}_{\textsf{fair}}, \boldsymbol{\theta}^{\text{PS}}_{\textsf{fair}})$  is independent of the choice of repeated deployment schema. 
\end{proposition}
Although the choice of repeated deployment schema does not affect the Fair-PS solution $(\mathcal{D}^{\text{PS}}_{\textsf{fair}}, \boldsymbol{\theta}^{\text{PS}}_{\textsf{fair}})$, it influences the convergence rate of the iterative algorithms in Theorem \ref{thm:fair_RRM}.  

\begin{theorem}[Convergence of Fair-RRM] \label{thm:fair_RRM} 
Algorithm~\ref{alg:fair_RRM} converges to a Fair-PS solution under the following deployment schemas: 
(i) under \textbf{conventional} deployment schema, it converges to the Fair-PS solution at a linear rate if $\epsilon(1 + \Tilde{\beta}/\gamma) < 1$; (ii) under \textbf{$k$-delayed} deployment schema, it converges to the Fair-PS solution at a linear rate for any $k$ if $\epsilon(1 + \Tilde{\beta}/\gamma) < 1-\epsilon$; (iii) under \textbf{delayed} deployment schema when $r = \log^{-1}\left( \frac{1}{\epsilon} \right) \log\left( \frac{\mathcal{W}_1(\mathcal{D}^{(0)}, \mathcal{D}^{(1)})}{\delta} \right) $, it converges to within a radius  $\delta$ of the Fair-PS solution (i.e., $\|\boldsymbol{\theta}^{(t)} -\boldsymbol{\theta}^{\text{PS}}_{\textsf{fair}}\|_2 \le \delta$ and $\mathcal{W}_1(\mathcal{D}(\boldsymbol{\theta}^{(t)}), \mathcal{D}^{\text{PS}}_{\textsf{fair}}) \le \delta$) in $\mathcal{O}(\log^2 \frac{1}{\delta})$ steps if $\epsilon(1 + \Tilde{\beta}/\gamma) < 1$.  
\end{theorem}

In App.~\ref{app:frerm}, we extend Algorithm \ref{alg:fair_RRM} to \textit{fair repeated empirical risk minimization (Fair-RERM)}. Although the theorems are for fairness mechanisms in Section \ref{sec:fair_mechanism}, the proofs can be easily extended to a general class of $\mathcal{L}_{\textsf{fair}}$ with convex and smooth fairness penalty terms (details in App. \ref{app:design}).

\paragraph{Fairness guarantee.} Thm. \ref{thm:fair_RRM} and App. \ref{app:frerm} show that repeatedly minimizing $\mathcal{L}_{\textsf{fair}}(\boldsymbol{\theta}; \mathcal{D}^{(t-1)}, \rho)$ on evolving data sequence can lead the system converging to a Fair-PS solution. Note that $\rho$ controls the strength of fairness and different $\rho$ could result in different $(\boldsymbol{\theta}^{\text{PS}}_{\textsf{fair}},\mathcal{D}^{\text{PS}}_{\textsf{fair}})$. In conventional supervised learning with static data distribution $\mathcal{D}$, it is trivial to see that larger $\rho$ will lead to a fairer solution \citep{martinez2020minimax}. However, in SDPP settings, the impact of $\rho$ on unfairness is less straightforward. Since both data distribution $\mathcal{D}^{\text{PS}}_{\textsf{fair}}$ and model $\boldsymbol{\theta}^{\text{PS}}_{\textsf{fair}}$ depend on $\rho$, analyzing how loss disparity would change as $\rho$ varies can be highly complex. However, we manage to study fair mechanisms in Eqn.~\eqref{eq:squared_group_loss_penalty} with \textit{group-level fairness penalty} to prove and quantify the fairness guarantee at the Fair-PS solution. We focus on a special case of user retention dynamics \citep{hashimoto2018fairness,Zhang_2019_Retention} with two groups $s\in\{a,b\}$, where the group distribution $\mathcal{D}_s$ is fixed but the fraction $p_s^{(t)}$ changes based on group loss $\mathcal{L}(
\boldsymbol{\theta}^{(t)};\mathcal{D}_s)$ during repeated risk minimization process.
\begin{assumption}\label{assumption:retention}
    The majority group always experiences lower expected loss, i.e., $\arg\max_{s\in\{a,b\}}p_s^t = \arg\min_{s\in\{a,b\}}\mathcal{L}(\boldsymbol{\theta}^{(t)};\mathcal{D}_s)$. 
    For two models deployed on population $\mathcal{D}^{(t)}$, the model with larger loss disparity $\triangle_{\mathcal{L}}^{(t)}$ leads to higher group fraction disparity $\triangle^{(t+1)}_{p}$ at time $t+1$.
\end{assumption}

Assumption \ref{assumption:retention} is natural: in applications such as recommendation systems, the minority group often suffers from higher loss and has a lower retention rate. Denote Fair-PS solution  $\left(\boldsymbol{\theta}^{\text{PS}}_{\textsf{fair}}\left(\rho\right), \mathcal{D}^{\text{PS}}_{\textsf{fair}}\left(\rho\right)\right)$ as a function of $\rho\geq 0$ and the original PS solution as $(\boldsymbol{\theta}^{\text{PS}}, \mathcal{D}^{\text{PS}})$. Let $\triangle^{\text{PS}}_{\textsf{fair},\mathcal{L}}(\rho)$ be the group loss disparity of $\boldsymbol{\theta}^{\text{PS}}_{\textsf{fair}}\left(\rho\right)$ evaluated on distribution $\mathcal{D}^{\text{PS}}_{\textsf{fair}}\left(\rho\right)$ and $\Delta _{\mathcal{L}}^{\text{PS}}$ be the group loss disparity of $\boldsymbol{\theta}^{\text{PS}}$ evaluated at $\mathcal{D}^{\text{PS}}$. We can first prove that a larger $\rho$ leads to stronger fairness in Thm. \ref{thm:fair_improvement}.

\begin{theorem}\label{thm:fair_improvement}
 Under Assumption \ref{assumption:retention}, $\triangle^{\text{PS}}_{\textsf{fair},\mathcal{L}}(\rho)$ is non-increasing in $\rho$. 
\end{theorem} 

Furthermore, we can quantify the fairness improvement.

\begin{theorem}\label{thm:quantified_fairness_improvement}
Denote $p_s^{\text{PS}}$ as the fraction of group $s$ at $\mathcal{D}^{\text{PS}}$ under retention dynamics.
Assume \(\mathcal{L}(\boldsymbol{\theta}; \mathcal{D})\) is twice continuously differentiable. For sufficiently small \(\rho > 0\), we have $$\Delta _{\mathcal{L}}^{\text{PS}} - \Delta _{\mathcal{L}, \textsf{fair}}^{\text{PS}}(\rho)
= 2\rho \, p_a^{\text{PS}} p_b^{\text{PS}} \, \Delta _{\mathcal{L}}^{\text{PS}} \cdot 
\boldsymbol{v}_{\textsf{fair}}^\top H^{-1} \boldsymbol{v}_{\textsf{fair}} + \mathcal{O}(\rho^2),$$
where
\(
\boldsymbol{v}_{\textsf{fair}} = \nabla_{\boldsymbol{\theta}} \mathcal{L}(\boldsymbol{\theta}^{\text{PS}}; \mathcal{D}_a)
- \nabla_{\boldsymbol{\theta}} \mathcal{L}(\boldsymbol{\theta}^{\text{PS}}; \mathcal{D}_b)\), and \(
H = \nabla_{\boldsymbol{\theta}}^2 \mathcal{L}(\boldsymbol{\theta}^{\text{PS}}; \mathcal{D}^{\text{PS}}).
\)
\end{theorem}
Since $H$ is positive definite due to the strong convexity of $\mathcal{L}_{\textsf{fair}}$ (Lemma \ref{lemma:fair_loss_gamma_convex}), Thm. \ref{thm:quantified_fairness_improvement} reveals the fairness improvement is positive when $\rho$ is sufficiently small even without Assumption \ref{assumption:retention}.

\section{Numerical results} \label{section:numerical}

This section empirically evaluates the proposed methods on synthetic and real data (including credit data \citep{perdomo_performative_2021,creditdata} and MNIST \citep{deng2012mnist}) under semi-synthesized performative shifts.
We run all experiments with multiple random seeds and visualize the standard errors. 

\paragraph{Performative Gaussian mean estimation.} We generate a synthetic dataset of $10000$ samples from $s\in\{a,b\}$ with initial group fractions $p_a^{(0)} = 0.3, p_b^{(0)} = 0.7$ and target values $y_a = 0.3 + \epsilon, y_b = 0.7 + \epsilon$, where $\epsilon \sim \mathcal{N}(0,0.05)$. At each round, the decision-maker estimates the mean of the current distribution as $\boldsymbol{\theta}^{(t)}$, and the loss function is the mean squared error with $l_2$ regularization. We consider user retention dynamics similar to \citet{hashimoto2018fairness} where $p^{(t+1)}_s = \frac{\mathcal{R}(s,t)}{\sum_{s' \in \{a,b\}}\mathcal{R}(s',t)}$ changes based on group-wise loss. Here $\mathcal{R}(s,t) = \left(1 - \sum_{s'\in\{a,b\}} p^{\min}_{s'}\right) \times \frac{1}{2}\left(p_s^t + \frac{\mathcal{L}(\boldsymbol{\theta}^{(t)};\mathcal{D}_{-s}^{(t)})}{\mathcal{L}(\boldsymbol{\theta}^{(t)};\mathcal{D}_{-s}^{(t)}) + \mathcal{L}(\boldsymbol{\theta}^{(t)};\mathcal{D}_{s}^{(t)})}\right) + p_s^{\min}$ where $p_s^{\min} = 0.02$ is the minimum group fraction, and $-s = \{a,b\}\setminus s$. We perform the mean estimation task empirically to train multiple linear regression models and compare \texttt{RERM} with \texttt{Fair-RERM}, including regularization methods with group-level penalty (\texttt{Fair-RERM-GLP}), sample-level penalty (\texttt{Fair-RERM-SLP}) and sample re-weighting method (\texttt{Fair-RERM-RW}). We perform $30$ rounds of empirical risk minimization on $7$ different random seeds. Fig. \ref{subfig:ldisp_gm} illustrates the evolution of group-wise loss disparity $\triangle_\mathcal{L}^{(t)}$, where higher fairness improvement is achieved with higher $\rho$. The sample re-weighting method (\texttt{Fair-RERM-RW}) seems to yield better fairness control than others in this setting. Fig. \ref{subfig:tradeoff_gm} shows the tradeoff between fairness (group-wise loss disparity) and the global loss. The ``stars" are produced by adjusting $\rho$ and demonstrate a ``Pareto-optimal" surface of each fairness mechanism.

\begin{figure}[t]
    \centering
    \begin{subfigure}[t]{0.43\textwidth}
        \includegraphics[width=\textwidth]{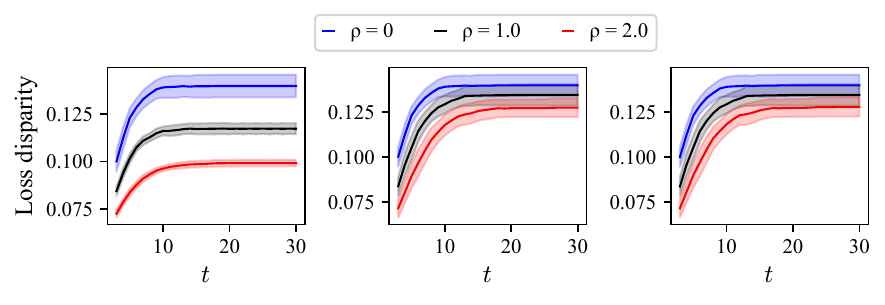}
        \caption{
            \parbox[t]{\linewidth}{
                Dynamics of loss disparity for Performative Gaussian mean estimation when $\rho \in \{0,1.0,2.0\}$. When $\rho = 0$, the policy reduces to standard \texttt{RERM}.
            }
        }
        \label{subfig:ldisp_gm}
    \end{subfigure}
    \begin{subfigure}[t]{0.43\textwidth}
        \includegraphics[width=\textwidth]{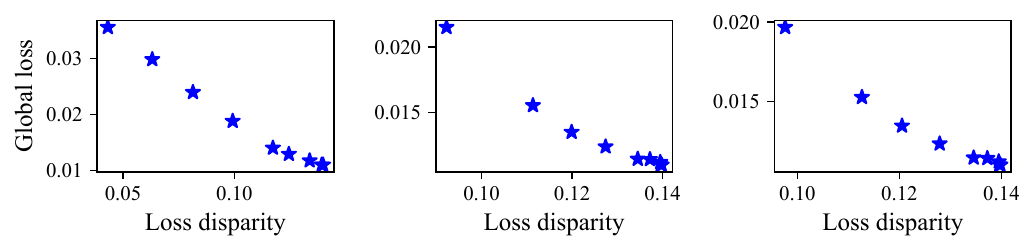}
        \caption{
            \parbox[t]{\linewidth}{
                Tradeoff between fairness and the global loss at the PS solution. ``Stars'' are produced by varying $\rho$.
            }
        }
        \label{subfig:tradeoff_gm}
    \end{subfigure}
    \caption{Results on Gaussian data: \texttt{Fair-RERM-RW} (left), \texttt{Fair-RERM-GLP} (middle), \texttt{Fair-RERM-SLP} (right).}
    \label{fig:gm}
\end{figure}

\paragraph{Credit data retention dynamics with strategic behaviors.} We use the \textit{Give Me Some Credit} data \citep{creditdata} consisting of features $X \in \mathbb{R}^{10}$ to measure individuals' creditworthiness $Y \in \{0,1\}$ \citep{perdomo_performative_2021, hu2022achieving}. We preprocessed the data similarly to \citet{perdomo_performative_2021} and divided individuals into two groups $s \in \{a,b\}$ based on the \texttt{age} attribute. Next, we assume there is a newly established credit rating agency comparing \texttt{RERM} with \texttt{Fair-RERM} with logistic classification models to predict individuals' creditworthiness. Similarly, we assume the group-wise loss in round $t$ affects the group fraction at $t+1$. Meanwhile, we assume there is a subset of features $X_s\in X$ which individuals can change strategically to $X_s'$ based on the current model parameters $\theta$. Specifically, $X_s' = X_s - \epsilon \cdot \theta_s$, where $\theta_s$ is the subset of $\theta$ with respect to $X_s$ and $\epsilon = 0.1$. With the dynamics, we can visualize the evolution of group-wise loss disparity and the tradeoff between fairness and the global loss in Fig. \ref{fig:credit}. All results demonstrate the effectiveness of our methods where \texttt{Fair-RERM-SLP} seems to be more effective.

\begin{figure}[t]
    \centering
    \begin{subfigure}[t]{0.43\textwidth}
        \includegraphics[trim=0.21cm 0.35cm 0.25cm 0.2cm,clip,width=\textwidth]{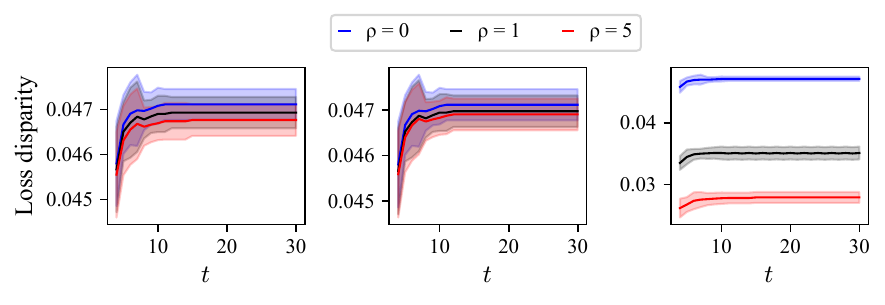}
        \caption{
            \parbox[t]{\linewidth}{
                Dynamics of group-wise loss disparity for Credit data when $\rho \in \{0,1.0,5.0\}$. When $\rho = 0$, the policy is just \texttt{RERM}.
            }
        }
        \label{subfig:ldisp_credit}
    \end{subfigure}
    \begin{subfigure}[t]{0.43\textwidth}
        \includegraphics[width=\textwidth]{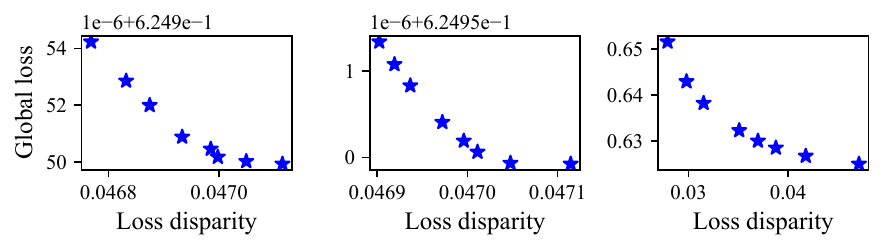}
        \caption{
            \parbox[t]{\linewidth}{
                Tradeoff between fairness and the global loss at the PS solution. ``Stars'' are produced by varying $\rho$.
            }
        }
        \label{subfig:tradeoff_credit}
    \end{subfigure}
    \caption{Results on Credit data: \texttt{Fair-RERM-RW} (left), \texttt{Fair-RERM-GLP} (middle), \texttt{Fair-RERM-SLP} (right).}
    \label{fig:credit}
\end{figure}

\paragraph{Additional experiments in App. \ref{app:addexp}.} Due to the page limit, we defer additional experiments to App.~\ref{app:converge}. We want to highlight that we perform experiments on MNIST data \citep{deng2012mnist} to test whether our fairness mechanisms can still be useful beyond the convex setting with a deep learning model. Remarkably, Fig.~\ref{subfig:ldisp_mnist} in App.~\ref{app:converge} verifies that \texttt{Fair-RERM} can still improve fairness at the PS solution. Moreover, Fig. \ref{subfig:tradeoff_mnist} in App.~\ref{app:converge} demonstrates that the Fair-PS solution may converge to a local stationary point better for both fairness and performative loss in non-convex PP settings. We also perform experiments on ACSIncome dataset \cite{ding2022retiringadultnewdatasets} (App. \ref{app:ACSIncome}) and show how our fairness mechanisms influence \textit{Equal opportunity} and \textit{Demographic Parity} (App. \ref{app:eqopt}).

\section{Conclusions \& limitations}\label{sec:conclusion}

Our work reveals unfairness issues of the PS solutions of PP. We propose novel fairness-aware algorithms to find Fair-PS solutions with the convergence holds under mild assumptions to facilitate trustworthy machine learning. However, the theory of this work does not cover non-convex settings and it still remains an open question for the future research.


\section*{Acknowledgements}
This work was funded in part by the National Science Foundation under award number IIS-2202699 and
IIS-2416895.

\bibliography{references}
\newpage
\appendix
\onecolumn
\setcounter{secnumdepth}{2}
\renewcommand{\thesubsection}{\thesection.\arabic{subsection}}

\section{Related works}\label{app:related}


\textbf{Performative Prediction.} Proposed by \citet{perdomo_performative_2021}, Performative prediction focuses on machine learning problems where the model deployments influence the data distribution, and the common application scenarios include strategic classification \citep{hardt_strategic_2015, raab2021unintended, xie2024algorithmic,xie2024learning,xie2025sprint,xie2024automating,xie2024non} and predictive policing \citep{ensign2018runaway}. \citet{perdomo_performative_2021} formulated the problem and proposed \textit{repeated risk minimization} (RRM) to ensure the convergence of model parameters to a stable point under a set of sufficient and necessary conditions. \citet{Mendler2020Stochastic} presented convergence results on stochastic performative optimization. Subsequent work either strived to relax the convergence condition \citep{mofakhami2023performative,zhao2022optimizing} or develop methods to ensure the convergence to performative optimal points under special cases \citep{miller_outside_2021, izzo2021performative}. \citet{brown2020performative} extended performative prediction to stateful settings where the current data distribution is determined by both the previous states and the deployed model. For more related works on PP, we refer to this survey \citep{hardt2023performative}. 

\textbf{Long-term fairness.} There is a large line of work on long-term fairness in machine learning, where we mainly refer to distributionally robust optimization (DRO) \citep{duchi2018learning, duchi2023distributionally} and soft fair regularizers \citep{Kami2011, zafar2017fairness, guldogan2022equal}. DRO refers to the optimization problem where the loss is minimized on the worst-case distribution around current data distribution \citep{duchi2018learning, duchi2023distributionally}, while soft fair regularizers are often used in traditional fair optimization frameworks to penalize unfairness. Only a few pieces of literature touched on fairness problems under performative prediction settings. \citep{milli2018social, humanip,Zhang_2022_ICML,Zhang_2020_Long_term} focused on long-term fairness settings under strategic classification similar to Example \ref{example:ps_sc}. \citet{zezulka2023performativity} pointed out fairness issues in performative prediction and formulated it within causal graphs. \citet{hu2022achieving} also used causality-based models to formulate and ensure short-term and long-term fairness in performative prediction, incurring strong assumptions and expensive computational costs.  \citet{Raab_Boczar_Fazel_Liu_2024} studied a special case where only the agents' retention rates are performative. \citet{peetpare2023long} proposed to use distributionally robust optimization (DRO) to promote fairness in performative prediction, but they proposed assumptions that are hard to verify and can not be derived from other commonly accepted assumptions in PP. Another contemporary work \citep{jia2024distributionally} developed more practical assumptions for DRO under PP settings, but they do not consider SDPP settings and did not discuss the fairness issues under PP settings. \citet{xue2024distributionally} applied DRO to find the PO solutions, but the algorithm convergence needs to be established in a case-by-case basis and with careful calibration. Meanwhile, they also did not discuss the fairness issues.

Finally, it may be worthwhile noting that many works of machine learning fairness exist, but most of them only consider the population as a static distribution (e.g., \citet{martinez2020minimax}). We refer to \citet{caton2024fairness} as a most recent comprehensive survey.

\section{Further elaboration on the concepts}

Here we include the notation table and a Venn diagram to help the readers better understand the technical details of our paper.

\subsection{Notation Table}

\begin{table}[!htb]
      \centering
        
        \begin{tabular}{p{0.2\textwidth} p{0.66\textwidth} }

        \toprule
        \textbf{Notation} & \textbf{Meanings} \\
        \midrule
        $\boldsymbol{\theta}$ & Model parameters
        \\

        \midrule
        $h_{\boldsymbol{\theta}}$ & The decision model parameterized by $\boldsymbol{\theta}$
        \\
        
        \midrule
        $d_{\boldsymbol{\theta}}$ & The fixed point distribution by repeatedly deploying $\boldsymbol{\theta}$
        \\
        
        \midrule
        $\Theta$ & The parameter space
        \\

        \midrule
        $Z$ & Data sample
        \\

        \midrule
        $X$ & Feature in the data sample
        \\

        \midrule
        $Y$ & Label in the data sample
        \\

        \midrule
        $\mathcal{D}$ & Sample distribution
        \\

        \midrule
        $\triangle(\mathcal{Z})$ & Support of sample distribution
        \\

        \midrule
        $\text{T}$ & Transition mapping
        \\

        \midrule
        $\text{T}^k$ & $k$-Delayed transition mapping
        \\

        \midrule
        $\text{T}^{DL}$ & Delayed transition mapping where $k = \lceil r \rceil + 1$
        \\

        \midrule
        $\ell$ & Loss function
        \\

        \midrule
        $s$ & Sensitive attribute, where group $s$ means the set of samples with attribute $s$
        \\

        \midrule
        $\mathcal{S}$ & The set of sensitive attributes
        \\
        
        \midrule
        $\mathcal{L}_s$ & The expected loss of group $s$
        \\

        \midrule
        $\mathcal{D}^{(t)}_s$ & Sample distribution of group $s$ at time $t$
        \\

        \midrule
        $p^{(t)}_s$ & Population fraction of group $s$ at time $t$
        \\

        \midrule
        $\mathcal{W}_1$ & The 1-Wasserstein distance
        \\

        \midrule
        $\triangle_{DP}, \triangle_{EO}, \triangle_{\mathcal{L}}$ & Demographic Parity, Equal Opportunity, Loss Disparity penalty
        \\

        \midrule
        $\mathcal{L}_{\textsf{fair}}$ & Fairness-aware objective
        \\

        \midrule
        $\mathcal{P}$ & The penalty term in fairness aware objective
        \\
        
        \midrule
        $\mathcal{B}$ & Distribution ball, used to limit the choices of re-weighting
        \\
        
        \midrule
        $\boldsymbol{l}^{(t)}$ & Loss-guided re-weighting vector for time step $t$
        \\

        \midrule
        $q^{(t)}_s$ & Group weight of group $s$ at time $t$ in the fair reweighting mechanism
        \\

        \midrule
        $\rho$ & Strength of fair mechanism
        \\
        
        \midrule
        $\gamma$ & Strong convexity coefficient of the loss
        \\
        
        \midrule
        $\beta$ & Joint smoothness coefficient of the loss
        \\

        \midrule
        $\epsilon$ & Joint sensitivity coefficient of the transition mapping
        \\

        \midrule
        $\Tilde{\beta}$ & Joint smoothness coefficient of the fairness-aware objective
        \\

        \midrule
        $\overline{\ell}$ & Maximum loss value given $\Theta$ and $\triangle(\mathcal{Z})$
        \\

        
		\bottomrule
  
		\end{tabular}
  \vspace{2mm}
		\caption{Summary of notation.}\label{table:notation}
\end{table}

\subsection{Venn diagram of concepts in PP}

We present a Venn diagram of different PP cases in Figure \ref{fig:venn}. Within the broad concept of PP, we specifically characterize 3 practical concepts in real-world problems, namely the (1) state-dependent performative shifts, (2) performative retention dynamics, and (3) performative distribution shifts. These concepts represent different dimensions to partition the PP problems.

\begin{figure}[ht]
\begin{center}
\centerline{\includegraphics[width=0.45\columnwidth]{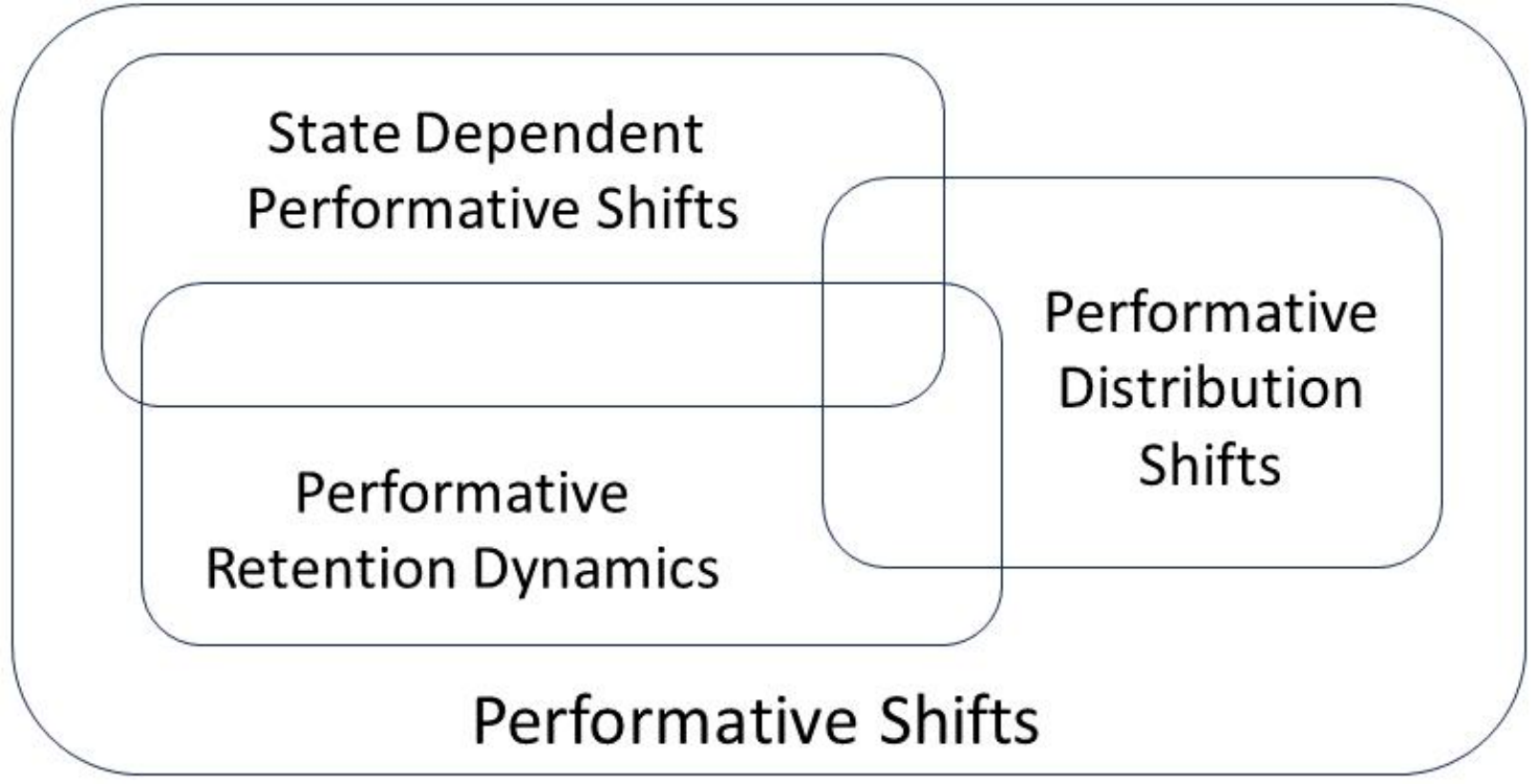}}
\caption{A Venn diagram of the PP cases.}
\label{fig:venn}
\end{center}
\end{figure}

\paragraph{Performative retention dynamics (PRD).} Previous works that study long-term fairness \citep{Zhang_2019_Retention, Zhang_2020_Long_term} mostly assume there to be model-dependent retention rate/population fraction $p_s$ changes, which can all be formulated as PRD problems. Recommendation systems are typical application scenarios of SDPP and PRD since the population fraction in the next iteration depends on both the current population and the current model. The better recommendation model captures a certain user group's interest, the more likely this user group will have high retention rate.

\begin{example} \label{example:retention} (User retention)
    Suppose two demographic groups $a,b$ each follow a static distribution $\mathcal{D}_a$ and $\mathcal{D}_b$, but their fraction in the population changes over time, and thus the overall data distribution is dynamic.
    Specifically, let $p_s^{min}$ be control variables to limit the minimum fraction of group $s$, and their retention rate is determined by a retention function $\pi_s(\boldsymbol{\theta}^{(t)})$. Then we have
    \begin{equation}
        p^{(t)}_s := \frac{p^{(t-1)}_s  \pi_s(\boldsymbol{\theta}^{(t)}) + p_s^{min}}{\sum_{s'} (p^{(t-1)}_{s'}  \pi_{s'}(\boldsymbol{\theta}^{(t)}) + p_{s'}^{min})}, ~~ \mathcal{D}^{(t)} = \sum_{s} p^{(t)}_s \mathcal{D}_s,
    \end{equation}
    which can be modeled by PP in the SDPP setting.
\end{example}

\paragraph{Performative distribution shifts (PDS).} In this paper, PDS specifically means in each group, the data distribution $\mathcal{D}_s$ shifts with the decision models, e.g., in strategic classification, users strategically manipulate their features for better decision outcomes. In recommendation systems, where creators strategically manipulate their features for better decision outcomes from the model \citep{hardt_strategic_2015}, and users' interests can be shifted by the recommended items.

We present a practical example of PP in addition to the one in the main article.

\begin{figure}
    \centering
    \begin{tikzpicture}[
			scale=0.12,
            > = stealth, 
            shorten > = 1pt, 
            auto,
            node distance = 1cm, 
            semithick 
        ]
        \tikzstyle{state}=[
            draw = none,
        ]
        \tikzstyle{hstate}=[
            draw = none,
            fill = gray
        ]
      \node (S) at  (0,14){$S$};
        \node[state] (X0) at  (-15,0){$X_0$};
        \node[state] (X1) at  (0,0){$X_1$};
        \node[state] (X2) at  (15,0){$X_2$};
        \node[hstate] (Y0) at  (-15,7){$Y_0$};
        \node[hstate] (Y1) at  (0,7){$Y_1$};
        \node[hstate] (Y2) at  (15,7){$Y_2$};
        \node[state] (A0) at  (-15,-7){$D_0$};
        \node[state] (A1) at  (0,-7){$D_1$};
        \node[state] (A2) at  (15,-7){$D_2$};
        \path[<-] (X0) edge node {} (Y0);
        \path[->] (X0) edge node {} (A0);
        \path[<-] (X1) edge node {} (Y1);
        \path[->] (X1) edge node {} (A1);
        \path[<-] (X2) edge node {} (Y2);
        \path[->] (X2) edge node {} (A2);
        \path[->] (Y0) edge node {} (Y1);
        \path[->] (Y1) edge node {} (Y2);
        \path[->] (A0) edge node {} (Y1);
        \path[->] (A1) edge node {} (Y2);
        \path[->] (S) edge node {} (Y0);
        \path[->] (S) edge node {} (Y1);
        \path[->] (S) edge node {} (Y2);
        \path[->] (S) edge node {} (A0);
        \path[->] (S) edge[bend left] node {} (A1);
        \path[->] (S) edge node {} (A2);
        \path[->] (S) edge node {} (X0);
        \path[->] (S) edge[bend right] node{} (X1);
        \path[->] (S) edge node {} (X2);
    \end{tikzpicture}
    \vskip -0.1in
    \caption{POMDP}
    \vskip -0.1in
  \label{fig:mdp}
\end{figure}
\begin{example} \label{example:POMDP} (Partially observable Markov decision process (POMDP) as PDS \citep{Zhang_2020_Long_term})
    The population fractions of the two groups $p_s$ are static, i.e., $\mathcal{D}^{(t)} = \sum_{s} p_s \mathcal{D}_s^{(t)}$, but the feature distribution changes with the deployed decision model, and such changes follow a POMDP as shown in Figure \ref{fig:mdp}.

    Specifically, within each group, there are two sub-populations, qualified and disqualified, e.g., the qualified sub-population in group $a$ satisfies $y=1, s = a$. Each sub-population has a static distribution, $\mathcal{D}_{s,y}$, but each group's distribution is dynamic and determined by
    \begin{equation}
        \mathcal{D}_s^{(t)} = \sum_y \alpha^{(t)}_{s,y} \mathcal{D}_{s,y}, ~~ \alpha^{(t)}_{s,y} := \mathbb{P}(Y^{(t)} = 1 | S=s),
    \end{equation}
    where $Y^{(t)}$ is the random variable of the agent's true label at time step $t$.
    The POMDP assumes the decision model and the sensitive attribute jointly determine the distribution of the true label distribution at the next time step, and we have
    \begin{equation}
        \alpha^{(t)}_{s,1} = \sum_{y = 0}^1 \alpha^{(t-1)}_{s,y}  \underbrace{\mathbb{P}(Y^{(t)} = 1 | Y^{(t-1)} = y, \boldsymbol{\theta} = \boldsymbol{\theta}^{(t)}, S=s)}_{\text{from POMDP}},
    \end{equation}
    and we can similarly write out the dynamics of all the $\alpha^{(t)}_{s,y}$ terms and thus the transition mapping $T$ as a PDS problem.
\end{example}

\subsection{Fair-RERM}\label{app:frerm}

Here we provide the algorithm for the Fair-RERM class and the results of its sample complexity.

\begin{algorithm}[H]
   \caption{Fair-RERM}
   \label{alg:fair_RERM}
\begin{algorithmic}
   \REQUIRE $t=0$,  $\mathcal{D}^{(0)}$, $\rho$,  $\boldsymbol{\theta}^{(0)}$, $\delta_{\theta}$, choose fair mechanism 
   \REPEAT
   \STATE Sample $\mathcal{Z}^{(t)}$ from $\mathcal{D}^{(t)}$
   \STATE $\boldsymbol{\theta}^{(t+1)} \leftarrow \arg \min_{\boldsymbol{\theta}}  \mathcal{L}_{\textsf{fair}}(\boldsymbol{\theta};\mathcal{Z}^{(t)},\rho)$
   \STATE Get the samples from $\mathcal{D}^{(t+1)}$ changing according to different schema
   \STATE $t \leftarrow t+1$
   \UNTIL{$\| \boldsymbol{\theta}^{(t)} - \boldsymbol{\theta}^{(t-1)}\|_2 \leq \delta_{\theta}$}
\end{algorithmic}
\end{algorithm}

\begin{theorem}[Convergence of fair-RERM] \label{thm:fair_RERM}
Suppose $\exists \alpha > 1, \mu > 0$ such that $\int_{\mathbb{R}^m} e^{\mu |x|^{\alpha}} Z dx$ is finite $\forall Z \in \triangle(\mathcal{Z})$. For a given convergence radius $\delta \in (0,1)$, take $n_t = \mathcal{O}\left(\frac{\log(t/p)}{(\epsilon(1+\Tilde{\beta}/\gamma)\delta)^m}\right)$ samples at $t$.  
    If $2 \epsilon(1 + \Tilde{\beta}/\gamma) < 1$, then with probability $1-p$, the iterates of fair-RERM are within a radius $\delta$ of the Fair-PS solution for $t \geq (1 - 2 \epsilon(1 + \Tilde{\beta}/\gamma) O(\log (1/\delta))$. 
\end{theorem}

\subsection{General guidance of designing fairness objectives under the SDPP setting}\label{app:design}

The key to designing a fairness objective while ensuring the convergence of the iterative algorithm is to ensure the convexity of the fairness penalty. Example \ref{example:group_loss_var} shows that group loss variance cannot preserve the convexity of $\mathcal{L}_{\textsf{fair}}$, so it fails to be a plausible objective. Generally, denote the penalty term as $\mathcal{P}(\boldsymbol{\theta}, \mathcal{D}, \rho)$ and the fair-regularized loss $\mathcal{L}_{\textsf{fair}} = \mathcal{L} + \mathcal{P}(\boldsymbol{\theta}, \mathcal{D}, \rho)$, we can derive a general sufficient condition for the convergence to a unique Fair-PS solution for Fair-RRM and Fair-RERM in Thm. \ref{thm:general}.

\begin{theorem}[Convergence of general fairness objectives]\label{thm:general}
    If $\mathcal{P}(\boldsymbol{\theta}, \mathcal{D}, \rho)$ is convex and $\beta_{P}$-smooth, and $\Tilde{\beta} = \beta_P + \beta$. Then same convergence results in Thm. \ref{thm:fair_RRM} and Thm. \ref{thm:fair_RERM} apply to $\mathcal{L}_{\textsf{fair}}$ with only the different $\Tilde{\beta}$ values dependent on $\mathcal{P}$.
\end{theorem}

Proof of Thm. \ref{thm:general} is almost identical to Thm. \ref{thm:fair_RRM} and Thm. \ref{thm:fair_RERM}, while the only difference lies in identifying that $\mathcal{L}_{\textsf{fair}}$ is $\Tilde{\beta}$-smooth.

However, it can still be highly non-trivial to design a fair objective in practice due to the following reasons: (i) it is hard to design an objective as interpretable as the ones we mentioned in Section \ref{sec:fair_mechanism}. The squared loss as a penalty term is well-motivated and directly penalizes the groups/samples with the higher losses; (ii) It becomes more difficult to obtain the concrete parameters for the convergence guarantee. For example, $\bar{\ell}$ in Section \ref{section:convergence} becomes the upper bound of $\frac{\partial g}{\partial \ell}$ and this may be hard to obtain and interpret. Overall, we see the formulations in Section \ref{sec:fair_mechanism} as more appropriate choices because loss disparity is applicable to both classification and regression settings. Also, many prominent works on recommendation systems (e.g., \citep{hashimoto2018fairness, Zhang_2019_Retention}) assumed the well-being of users are determined by the loss.

\subsection{Only measure fairness at PS solution}\label{app:mfair}

We provide an example to illustrate that fairness is not guaranteed even if we ensure instantaneous fairness without stability being achieved. 

\begin{example} (Fairness violation due to performative shifts) Consider a binary strategic classification problem with two groups $a,b$ and is a threshold model $h_{\theta}(X) = \mathbf{1}\{X \geq \theta\}$. $p_a, p_b$ are fixed and both groups satisfy $\mathbb{P}(Y^{(t)}=1) = \min\{0,\max\{X^{(t)},1\}\}$, and the transitions satisfy $X^{(t)} = X^{(0)} + (0.1+0.1 \cdot \boldsymbol{1}\{s=a\})\theta^{(t)}$ but do not result in any changes in labels. If at $t=0$, both groups have features $X^{(0)}$ subject to uniform distribution on $[0,1]$, then any $\theta = 0.5$ satisfies accuracy parity. However, at $t=1$, group $a$ has a shifted distribution, resulting in $\theta = 0.5$ to be unfair. 
\end{example}

\subsection{Additional corollaries on PS-fairness guarantee} \label{app:fairness_guarantee}

Below we provide an example corollary derived from Theorem \ref{thm:fair_improvement}, and we can similarly derive the PS-Fairness guarantee under other choices of algorithms and mechanisms.

\begin{corollary}
    In Algorithm \ref{alg:fair_RRM}, under we can obtain $\rho_{max}$ under different fairness mechanisms, such that $\epsilon(1 + \Tilde{\beta}_{max}/\gamma) = 1$, where 
    $$\Tilde{\beta}_{max}:= \begin{cases}
     (2 \rho_{max} \overline{\ell} + 1)\beta, ~\text{for regularization method \eqref{eq:squared_group_loss_penalty} or \eqref{eq:fair_loss_sample_level} }\\
     (\rho_{max} \overline{\ell} + 1)\beta,~\text{for sample re-weighting method \eqref{eq:fair_reweight} }
    \end{cases}$$
    Then we define
    \begin{equation}
    \triangle^{\text{PS}}_{\textsf{max-fair},\mathcal{L}}:= \lim_{\delta \rightarrow 0^+} \triangle^{\text{PS}}_{\textsf{fair},\mathcal{L}}(\rho_{max} - \delta),
    \end{equation}
    and we can guarantee that for fairness constraint such that $\triangle^{\text{PS}}_{\textsf{fair-tolerance},\mathcal{L}} \geq \triangle^{\text{PS}}_{\textsf{max-fair},\mathcal{L}}$, we can find a $\rho < \rho_{max}$ that produce a Fair-PS solution with fairness penalty below the constraint, i.e., $\triangle^{\text{PS}}_{\textsf{fair},\mathcal{L}}(\rho) < \triangle^{\text{PS}}_{\textsf{fair-tolerance},\mathcal{L}}$
    
\end{corollary}

\section{Additional experiments}\label{app:addexp}

All experiments are run on the CPU of a Macbook Pro with Apple M1 Pro chip and 16GB memory. We use Python 3.9 for all tasks. For datasets, Yann LeCun and Corinna Cortes hold the copyright of MNIST, which is a derivative work from original NIST datasets. MNIST dataset is made available under the terms of the Creative Commons Attribution-Share Alike 3.0 license; the credit dataset has been published by \citep{creditdata}; In App. \ref{app:converge}, we visualize the convergence of performative loss for all the above experiments. We also visualize the evolution of minority group fractions for them. In App. \ref{app:gcls}, we perform experiments on a performative Gaussian data classification task. In App. \ref{app:multiple_exp}, we conduct experiments on 3 groups. In App. \ref{app:e5}, we visualize the unconvergence of the fairness penalty using group loss variance as a complementary to Example \ref{example:group_loss_var}; Finally, we examine the \textbf{k-delayed} schema on all our fairness mechanisms on Credit data in App. \ref{app:k_delayed} and obtain the similar results.

\subsection{Complementary results of main experiments}\label{app:converge}

\begin{figure}[H]
    \centering
    \begin{subfigure}[b]{0.43\textwidth}
        \includegraphics[width=\textwidth]{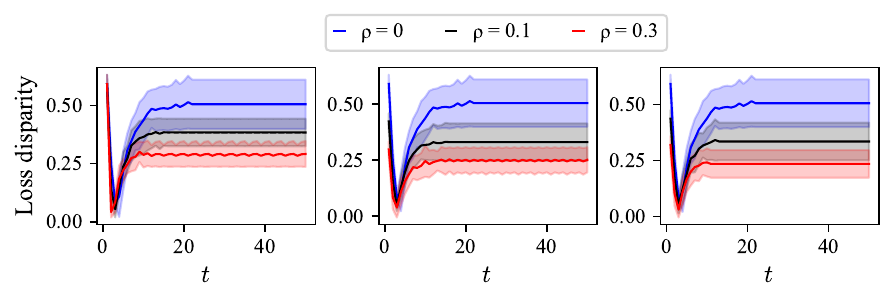}
        \caption{Dynamics of group-wise loss disparity for \textbf{MNIST} data when $\rho \in \{0,0.1,0.3\}$ When $\rho = 0$, the policy is just \texttt{RERM}.}
        \label{subfig:ldisp_mnist}
    \end{subfigure}
    \hspace{0.5cm}
    \begin{subfigure}[b]{0.43\textwidth}
        \includegraphics[width=\textwidth]{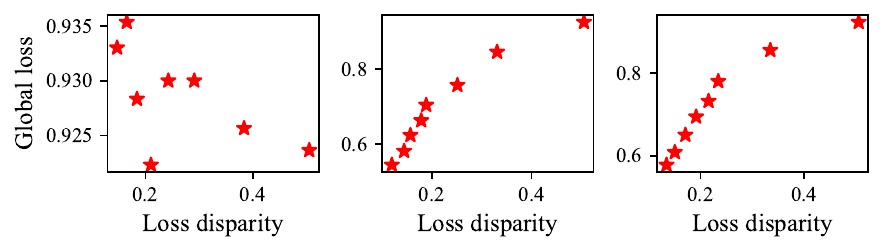}
        \caption{Tradeoff \textbf{not necessarily exists} at the PS solution. ``Stars" are produced by varying $\rho$.}
        \label{subfig:tradeoff_mnist}
    \end{subfigure}
    \caption{MNIST results: \texttt{Fair-RERM-RW} (left), \texttt{Fair-RERM-GLP} (middle), \texttt{Fair-RERM-SLP} (right).}
    \label{fig:mnist}
\end{figure}

\paragraph{Performative MNIST classification.} MNIST data \citep{deng2012mnist} consists of handwritten digits images. We randomly select 1000 images as the training set with 2 groups, where the first contains digits from $0$ to $4$ and the second includes digits from $5$ to $9$. We consider the same group fraction dynamics as before with $p_s^{\min} = 0.1$ for both groups. Using a 2-layer MLP for risk minimization, we perform \texttt{RERM} and \texttt{Fair-RERM} for $50$ rounds. Fig.~\ref{subfig:ldisp_mnist} verifies that \texttt{Fair-RERM} can still improve fairness at the PS solution. However, this set of experiments entails non-convexity because of the use of a neural network and there is no guarantee a PS solution will exist. In Fig. \ref{subfig:ldisp_mnist}, slight instability exists for the black and red lines ($\rho = 0.3$) of \texttt{Fair-RERM-RW} and \texttt{Fair-RERM-GLP}. \textbf{Notably}, Fig. \ref{subfig:tradeoff_mnist} demonstrates that the tradeoff between fairness and performative loss does not necessarily exist. The Fair-PS solution may converge to a local stationary point better for both fairness and performative loss. Although it is challenging to theoretically explain this phenomenon, the results suggest that our fairness mechanisms may still be useful in practice with deep learning models.

Firstly, provide visualizations of the convergence of performative loss for all experiments of the main paper in Fig. \ref{fig:converge}. Importantly, Fig. \ref{subfig:mnist_nonconverge} demonstrates the unconvergence of \texttt{Fair-RERM-RW} and \texttt{Fair-RERM-GLP} when $\rho = 0.3$. Secondly, we provide the evolution of minority group fractions for all three datasets in Fig. \ref{fig:prate}. These results clearly demonstrate our fairness mechanisms are effective in relieving polarization effects. Finally, we visualize the group-wise loss at Fair-PS solutions when $\rho$ increases in Fig. \ref{fig:groupwise}. It is worth noting that the fairness mechanisms do not always result in the majority group having higher loss and the minority group having lower loss. It is possible that the Fair-PS solutions are better for both groups (the right plot of Credit data and the middle/right plots of MNIST data). 

\begin{figure}[h]
    \centering
    \includegraphics[width=0.7\textwidth]{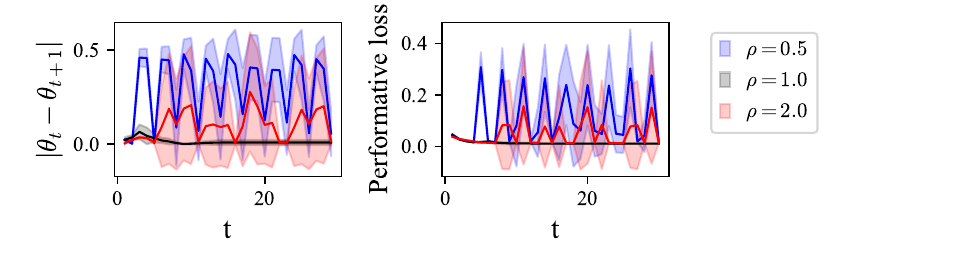}
    \caption{Dynamics of $\theta$ and performative loss in Gaussian mean estimation setting while using group loss variance as the fairness mechanism and $\rho \in \{0.5, 1.0, 2.0\}$. The dynamics demonstrate non-convergence of this mechanism.}
    \label{fig:e5}
\end{figure}

\begin{figure}[h]
    \centering
    \vspace{0.2cm}
    
    \begin{subfigure}[t]{0.245\textwidth}
        \includegraphics[width=\textwidth]{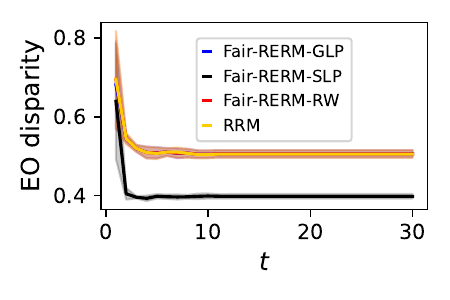}
        \caption{\textbf{Credit} (EO)}
        \label{subfig:eostrat}
    \end{subfigure}
    \hfill
    \begin{subfigure}[t]{0.245\textwidth}
        \includegraphics[width=\textwidth]{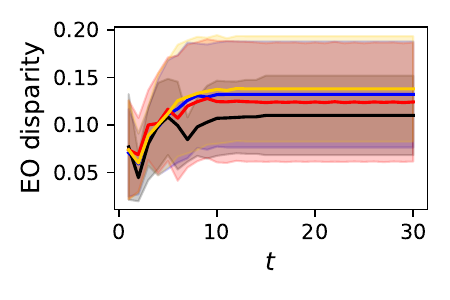}
        \caption{\textbf{Gaussian} (EO)}
        \label{subfig:eogclf}
    \end{subfigure}
    \hfill
    \begin{subfigure}[t]{0.245\textwidth}
        \includegraphics[width=\textwidth]{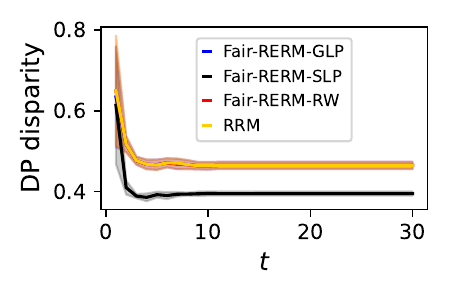}
        \caption{\textbf{Credit} (DP)}
        \label{subfig:dpstrat}
    \end{subfigure}
    \hfill
    \begin{subfigure}[t]{0.245\textwidth}
        \includegraphics[width=\textwidth]{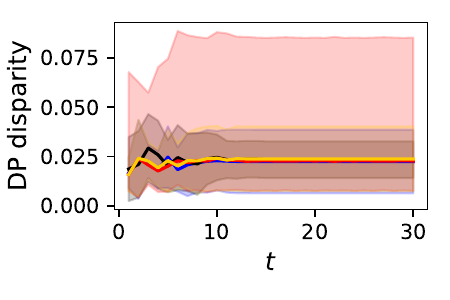}
        \caption{\textbf{Gaussian} (DP)}
        \label{subfig:dpgclf}
    \end{subfigure}
    
    \caption{Equal opportunity (EO) and demographic parity (DP) disparities on \textbf{Credit} and \textbf{Gaussian classification} datasets with $\rho=0.3$.}
    \label{fig:eo_dp_combined}
\end{figure}

\begin{figure}[h]
    \centering
    \begin{subfigure}[b]{0.75\textwidth}
        \includegraphics[width=\textwidth]{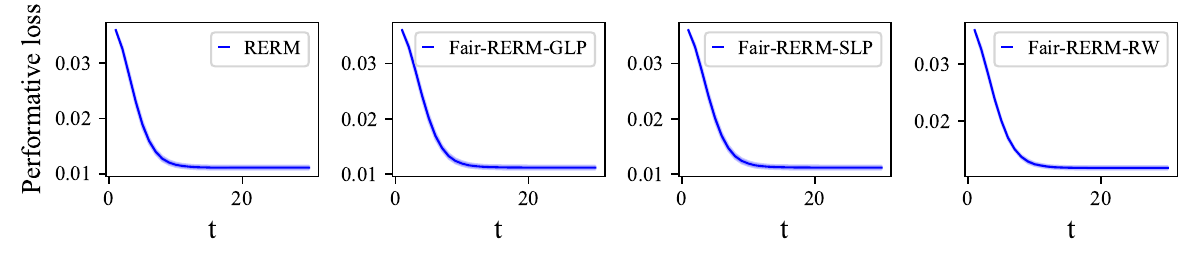}
        \caption{Dynamics of the performative loss for \textbf{Gaussian mean estimation task} when $\rho = 0.1$: \texttt{Fair-RERM-RW} (left plot), \texttt{Fair-RERM-GLP} (middle plot), \texttt{Fair-RERM-SLP} (right plot).}
        \label{subfig:gm_converge}
    \end{subfigure}
    \vspace{0.2cm}
    \begin{subfigure}[b]{0.75\textwidth}
        \includegraphics[width=\textwidth]{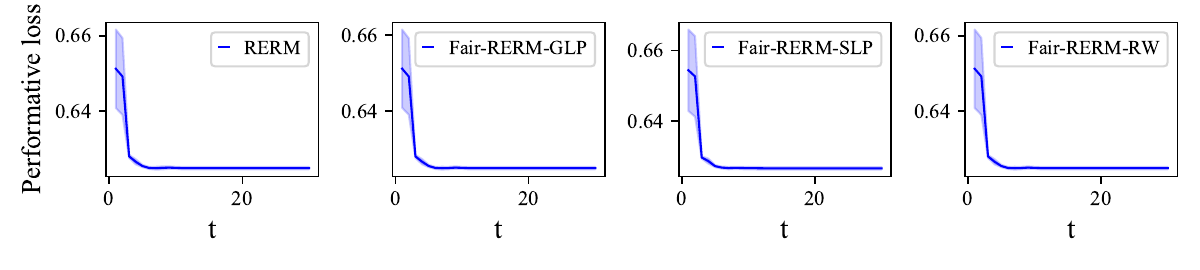}
        \caption{Dynamics of the performative loss for \textbf{Credit data task} when $\rho = 0.1$: \texttt{Fair-RERM-RW} (left plot), \texttt{Fair-RERM-GLP} (middle plot), \texttt{Fair-RERM-SLP} (right plot).}
        \label{subfig:credit_converge}
    \end{subfigure}
    \vspace{0.2cm}
    \begin{subfigure}[b]{0.75\textwidth}
        \includegraphics[width=\textwidth]{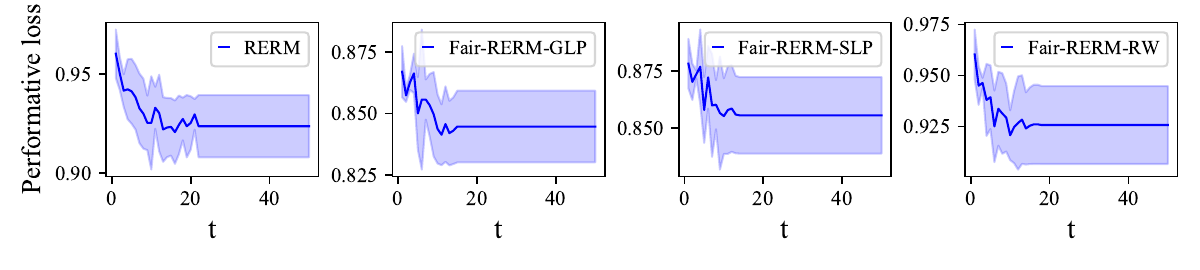}
        \caption{Dynamics of the performative loss for \textbf{MNIST classification task} when $\rho = 0.1$: \texttt{Fair-RERM-RW} (left plot), \texttt{Fair-RERM-GLP} (middle plot), \texttt{Fair-RERM-SLP} (right plot).}
        \label{subfig:mnist_converge}
    \end{subfigure}
    \vspace{0.2cm}
    \begin{subfigure}[b]{0.75\textwidth}
        \includegraphics[width=\textwidth]{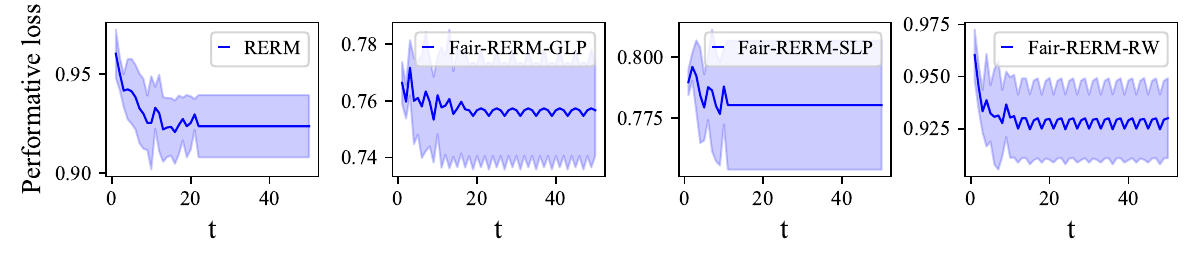}
        \caption{Dynamics of the performative loss for \textbf{MNIST classification task} when $\rho = 0.3$: \texttt{Fair-RERM-RW} (left plot), \texttt{Fair-RERM-GLP} (middle plot), \texttt{Fair-RERM-SLP} (right plot).}
        \label{subfig:mnist_nonconverge}
    \end{subfigure}
    \caption{Peformative loss dynamics under different $\rho$ }
    \label{fig:converge}
\end{figure}

\begin{figure}[h]
    \centering
    \includegraphics[width=0.6\linewidth]{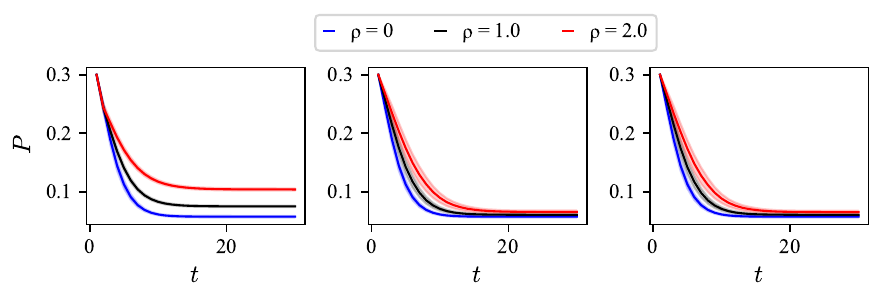}
    \includegraphics[width=0.6\linewidth]{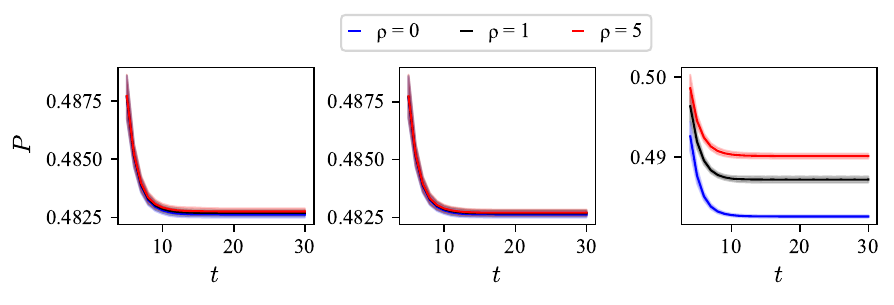}
    \includegraphics[width=0.6\linewidth]{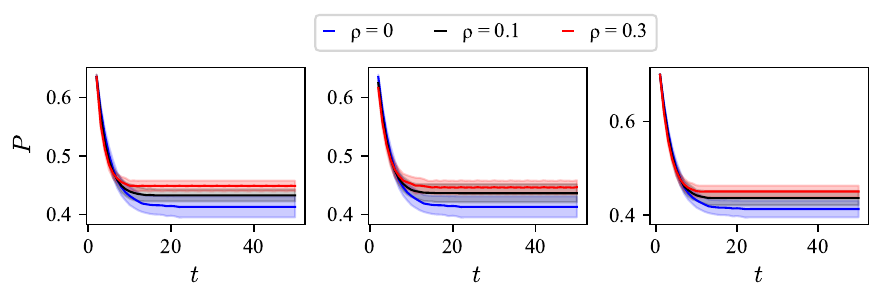}
    \caption{The evolutions of minority group fractions: top (Gaussian mean estimation); middle (Credit data); bottom (MNIST data).}
    \label{fig:prate}
\end{figure}

\begin{figure}[h]
    \centering
    \includegraphics[width=0.6\linewidth]{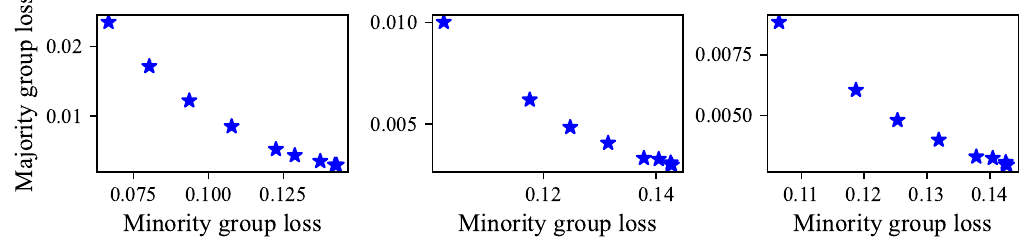}
    \includegraphics[width=0.6\linewidth]{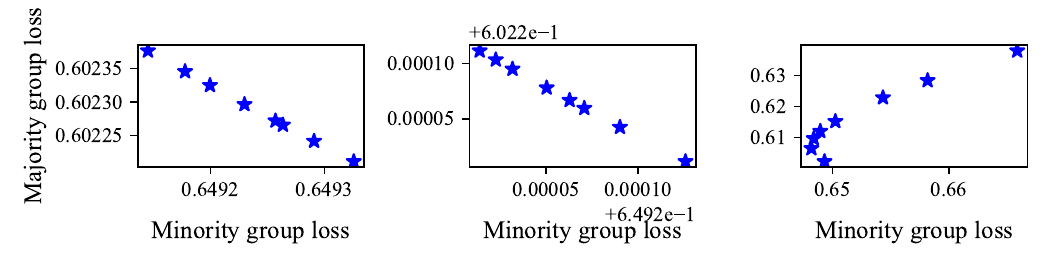}
    \includegraphics[width=0.6\linewidth]{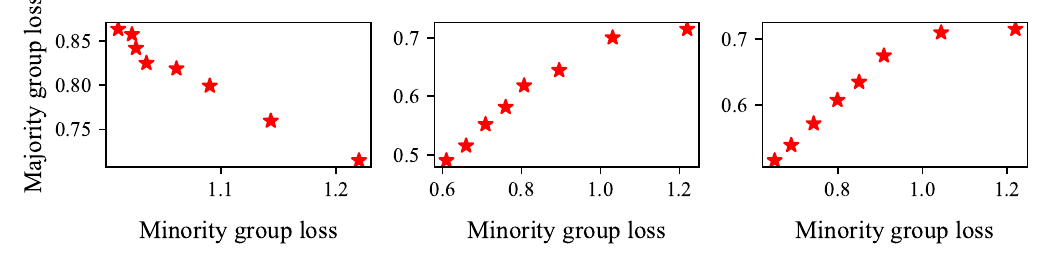}
    \caption{Group-wise loss at Fair-PS solutions while varying $\rho$: top (Gaussian mean estimation); middle (Credit data); bottom (MNIST data).}
    \label{fig:groupwise}
\end{figure}

\begin{figure}[h]
\centering
\includegraphics[width=0.8\textwidth]{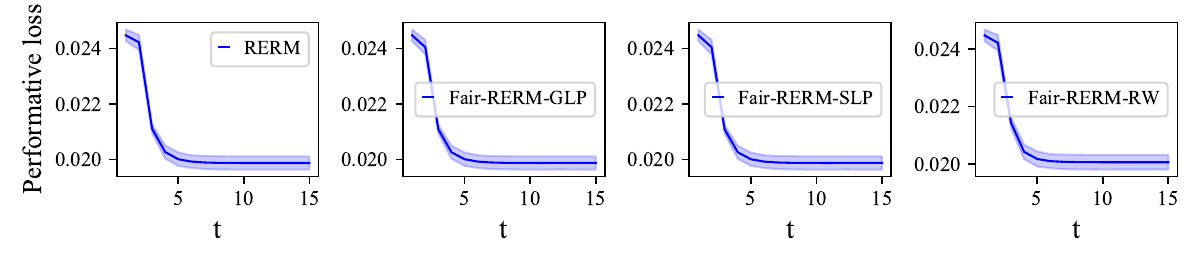}
\vspace{-0.5cm}
\caption{Dynamics of the performative loss for \textbf{Gaussian mean estimation task with multiple groups} when $\rho = 0.3$: \texttt{Fair-RERM-RW} (left plot), \texttt{Fair-RERM-GLP} (middle plot), \texttt{Fair-RERM-SLP} (right plot).}
\label{fig:mgm_converge}
\end{figure}

\begin{figure}[h]
    \centering
        \includegraphics[width=0.6\textwidth]{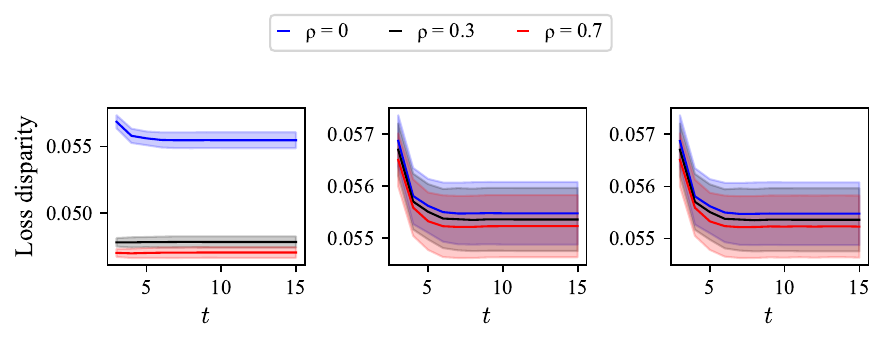}
        \caption{Dynamics of the performative loss disparity under different $\rho$: \texttt{Fair-RERM-RW} (left plot), \texttt{Fair-RERM-GLP} (middle plot), \texttt{Fair-RERM-SLP} (right plot).}
        \label{fig:ldisp_mgm}
\end{figure}

\subsection{Gaussian mean estimation with multiple groups}\label{app:multiple_exp}

In this section, we present the experimental results on the Gaussian mean estimation task with 3 groups $s \in \{a, b, c\}$ where $p_a^{(0)} = 0.15, p_b^{(0)} = 0.25, p_c^{(0)} = 0.6$ with different target values. We let $y_a = 0.3 + \epsilon, y_b = 0.5 + \epsilon, y_c = 0.7 + \epsilon$ where $\epsilon \sim \mathcal{N}(0,0.05)$.  The minimum group fraction for each group is $p^{min} = \{0.1,0.1,0.1\}$. At each time $t$, denote $\mathcal{R}(s,t) = \left(1 - \sum_{s'\in\{a,b\}} p^{min}_{s'}\right) \times \frac{1}{2}\left(p_s^t + \frac{\mathcal{L}(\boldsymbol{\theta}^{(t)};\mathcal{D}_{i_s}^{(t)})}{\sum_{s'} \mathcal{L}(\boldsymbol{\theta}^{(t)};\mathcal{D}_{s'}^{(t)})}\right) + p_s^{min}$. $-s$ is the group other than $s$, while $i_s$ is found by first ranking all group-wise losses in ascending order and then selecting the $(n+1-i)-th$ largest loss if group $s$ has the $i-th$ largest loss. This retention dynamic is a direct generalization of the two-group case. Next, we show the retention of group $p^{(t+1)}_s$ at time $t+1$ is shown as Eqn.~\eqref{eq:exp_retention_multi}, resulting in the new data distribution $D^{(t+1)}$.

\begin{equation}\label{eq:exp_retention_multi}
        p^{(t+1)}_s := \frac{\mathcal{R}(s,t)}{\sum_{s' \in \{a,b\}}\mathcal{R}(s,t)}
\end{equation}

We perform the mean estimation task for the whole data distribution with both \texttt{RERM} and \texttt{Fair-RERM} using the linear regression model. For \texttt{Fair-RERM}, we use \texttt{RERM}, \texttt{Fair-RERM-RW} and \texttt{Fair-RERM} to perform $30$ rounds of risk minimization where the loss is mean squared error (MSE), and name them as \texttt{Fair-RERM-GLP} (group-level penalty), \texttt{Fair-RERM-SLP} (sample-level penalty) and \texttt{Fair-RERM-RW} (reweighting). We first verify the convergence of performative loss of the above methods in Fig. \ref{fig:mgm_converge} where all \texttt{Fair-RERM} methods have $\rho=0.1$. The results demonstrate all methods successfully converge to the PS point.

Next, we verify the effectiveness of \texttt{Fair-RERM} by visualizing the dynamics of performative loss disparity between the groups while varying $\rho$ in Fig. \ref{fig:ldisp_mgm}, where the higher $\rho$ still results in lower loss disparity at the stable point, revealing that \texttt{Fair-RERM-RW} and \texttt{Fair-RERM} are effective.

\subsection{ACSIncome data classification.}\label{app:ACSIncome}

we conduct an additional set of experiments on the ACSIncome-CA dataset \cite{ding2022retiringadultnewdatasets}, a pre-processed version of the widely used Adult dataset from recent fairness literature. Specifically, the prediction target is whether an individual's annual income exceeds $50,000$ USD. We use all numerical features (AGEP, SCHL, WKHP) for RERM and Fair-RERM, with AGEP serving as the sensitive attribute (Age > 35). All other settings align exactly with the Credit data experiments in Section \ref{section:numerical}. We show the dynamics of group-wise loss disparity and evolution of the majority group fraction when $\rho = \{0,1.0,3.0\}$.

\begin{figure}[h]
    \centering
        \includegraphics[trim=0.21cm 0.35cm 0.25cm 0.2cm,clip,width=0.8\textwidth]{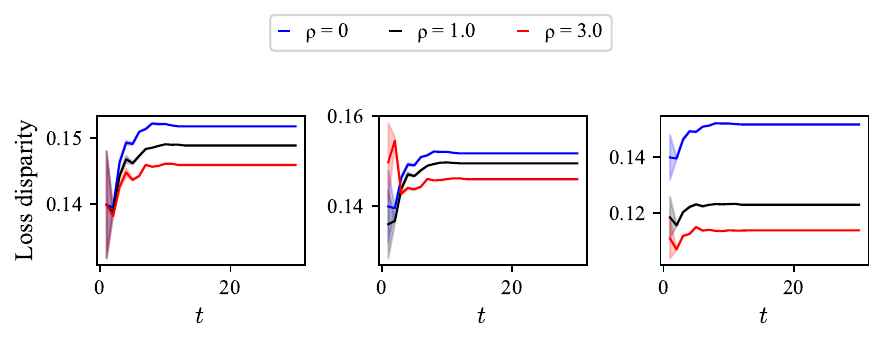}
        \caption{Dynamics of group-wise loss disparity for \textbf{ACSIncome-CA} data when $\rho \in \{0,1.0,3.0\}$. When $\rho = 0$, the policy is just \texttt{RERM} It is clear that when $\rho$ becomes larger, the group-wise loss disparity becomes lower.}
        \includegraphics[trim=0.21cm 0.35cm 0.25cm 0.2cm,clip,width=0.8\textwidth]{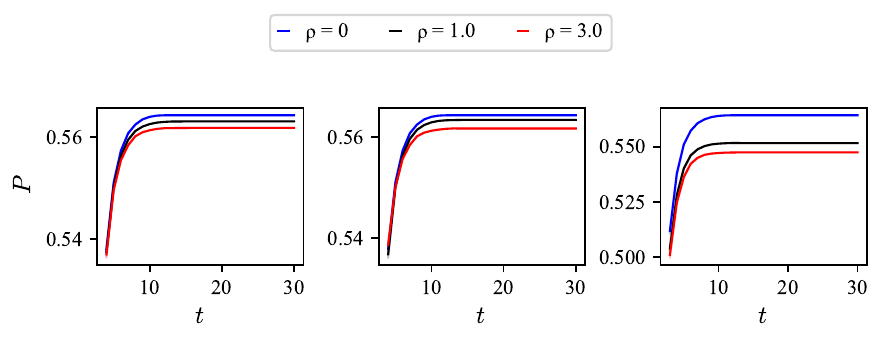}
        \caption{The evolution of majority group for \textbf{ACSIncome-CA} data when $\rho \in \{0,1.0,3.0\}$. When $\rho = 0$, the policy is just \texttt{RERM}. It is clear that when $\rho$ becomes larger, the group fraction becomes lower.}
    \label{fig:income_ldisp}
\end{figure}

\subsection{Performative Gaussian data classification.}\label{app:gcls}

We use a synthetic dataset consisting of $1000$ samples from 2 demographic groups $s \in \{a,b\}$ with features $X = \{X_1, X_2\}$ and labels $Y \in \{0,1\}$. For group $a$, $Y = \mathbf{1}\{x_1 - 0.5x_2 \ge 0.5\}$. For group $b$, $Y = \mathbf{1}\{0.5x_1 + 0.5x_2 \ge 0.5\}$. The initial group fraction and the retention mapping are the same as in the previous experiment. Using a logistic classification model for risk minimization, we first verify the convergence of performative loss in Fig. \ref{fig:gclf_converge} and the effectiveness of \texttt{Fair-RERM-RW} and \texttt{Fair-RERM} in Fig.\ref{subfig:ldisp_gclf} and Fig. \ref{subfig:prate_gclf}. All plots demonstrate the same trends as the previous experiment.
\begin{figure}[h]
    \centering
    \begin{subfigure}[b]{0.47\textwidth}
        \includegraphics[width=\textwidth]{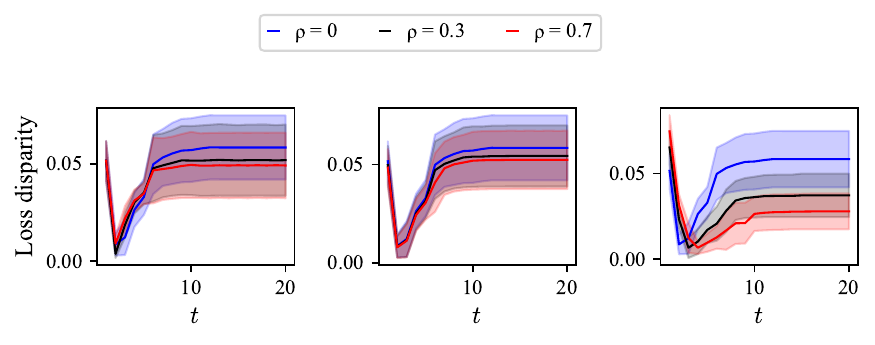}
        \caption{Dynamics of the loss disparity under different $\rho$: \texttt{Fair-RERM-RW} (left plot), \texttt{Fair-RERM-GLP} (middle plot), \texttt{Fair-RERM-SLP} (right plot).}
        \label{subfig:ldisp_gclf}
    \end{subfigure}
    \hfill
    \begin{subfigure}[b]{0.47\textwidth}
        \includegraphics[width=\textwidth]{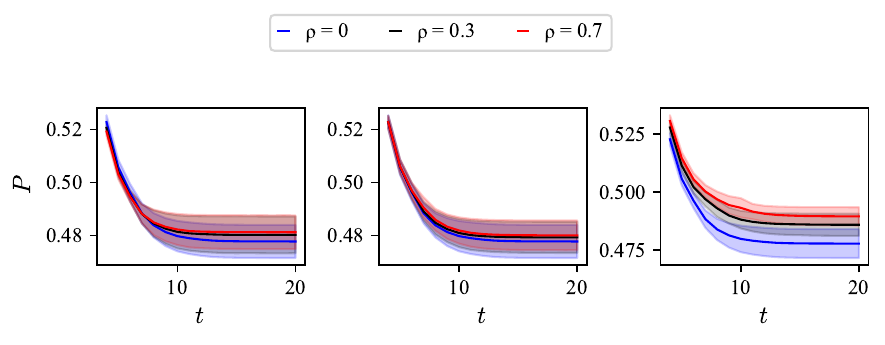}
        \caption{Dynamics of the minority group fraction under different $\rho$: \texttt{Fair-RERM-RW} (left plot), \texttt{Fair-RERM-GLP} (middle plot), \texttt{Fair-RERM-SLP} (right plot).}
        \label{subfig:prate_gclf}
    \end{subfigure}
    \caption{Fairness dynamics of Performative Gaussian Classification where $\rho \in \{0,0.3,0.7\}$ ($\rho = 0$ is equivalent to \texttt{RERM})}
    \label{fig:gclf}
\end{figure}

\begin{figure}[h]
    \centering
    \includegraphics[width=0.9\textwidth]{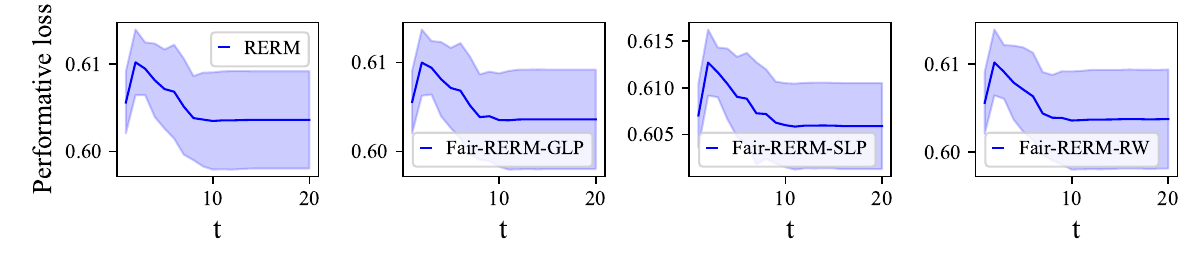}
    \vspace{-0.5cm}
        \caption{Dynamics of the performative loss for \textbf{Gaussian data classification task} when $\rho = 0.3$: \texttt{Fair-RERM-RW} (left plot), \texttt{Fair-RERM-GLP} (middle plot), \texttt{Fair-RERM-SLP} (right plot).}
    \label{fig:gclf_converge}
\end{figure}

\begin{figure}[h]
    \centering
    \includegraphics[width=0.9\textwidth]{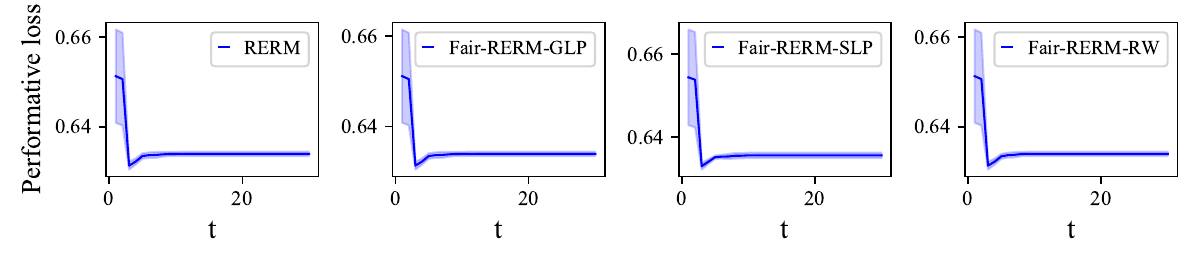}
    \vspace{-0.5cm}
        \caption{Dynamics of the performative loss for \textbf{k-delayed Credit task} when $\rho = 0.3$: \texttt{Fair-RERM-RW} (left plot), \texttt{Fair-RERM-GLP} (middle plot), \texttt{Fair-RERM-SLP} (right plot).}
    \label{fig:kcredit_converge}
\end{figure}

\subsection{Gaussian mean estimation Using group loss variance}\label{app:e5}

In Example \ref{example:group_loss_var}, we already prove that group loss variance as a fair penalty term violates the strong convexity assumption. In this section, we show its unconvergence while performing the Gaussian mean estimation task under the same setting mentioned in Section \ref{section:numerical} when $\rho = \{0.5, 1.5, 2.0\}$. The dynamics of $\|\theta_t - \theta_{t-1}\|$ as well as performative loss are shown as follows. Fig. \ref{fig:e5} shows the failure of convergence, which is contradicted to Fig. \ref{subfig:gm_converge} where we show the convergence of \texttt{Fair-RERM}.

\subsection{K-delayed RERM on Credit Data}\label{app:k_delayed}

We examine the effectiveness of k-delayed RERM schema on Credit Data \citep{creditdata} which is the same dataset used by \citet{brown2020performative}. Specifically, we use $k = 3$ and $\epsilon = 1$ to form equal-sized agent groups uniformly in the whole population to let each subgroup best respond with different speed (Example 1 in \citet{brown2020performative}). We also let the retention dynamic be k-delayed, adding additional challenges compared to the previous work. We examine $\rho \in \{0.3, 0.7\}$ while all other settings are the same as the Credit Data experiment in Section \ref{section:numerical}. 

Then we first verify the convergence of performative loss in Fig. \ref{fig:gclf_converge} and the effectiveness of k-delayed \texttt{Fair-RERM-RW} and \texttt{Fair-RERM} in Fig.\ref{subfig:ldisp_gclf} and Fig. \ref{subfig:prate_gclf}. All plots demonstrate the same trends as the previous experiment with RRM schema.

\subsection{Complementary results on Equal Opportunity and Demographic Parity}\label{app:eqopt}

Though group-wise loss disparity has been a widely used fairness metric for both regression and classification tasks, we also provide a complementary experimental result on how our fairness mechanism influences \texttt{Equal opportunity} \citep{hardt2016equality}. Specifically, we visualize the evolution of the true positive rate disparity between two groups until convergence to the Fair-PS point for the two classification tasks, and we visualize the evolution of the positive rate disparity. The results are in Fig. \ref{fig:eo_dp_combined}, which empirically demonstrate the effectiveness of all fairness mechanisms in improving equal opportunity and demographic parity.

\FloatBarrier

\newpage
\section{Proofs for the theoretical results}

\subsection{Proof of Lemma \ref{lemma:SDPP}}\label{app:proof_SDPP}

\begin{proof}
    The first part of Lemma \ref{lemma:SDPP} is just Thm. 8 in \citet{brown2020performative}. For the second part, we prove that when (ii) and (iii) are satisfied, there exists a non-convex loss function $\ell$ where the iterative algorithms fail to converge by constructing an example:

    Consider the cubic loss $\ell((x, y); \theta) = \beta y \theta^3$, for $\theta \in [-1, 1]$. 
    This objective is $\beta$-jointly smooth but not convex. Let the distribution of $Y$ according to $\text{T}(\boldsymbol{\theta}, \mathcal{D})$ be a point mass at $\epsilon \theta$, and let the distribution of $X$ be invariant. Then this distribution is $\epsilon$-sensitive.
    
    Then the decoupled performative risk has the following form: 
    $\text{DPR}(\theta, \varphi) = \epsilon \beta \theta^3 \varphi$. Then if we initialize RRM at any point other than $0$ with any starting distribution and $\boldsymbol{\theta}^{(0)} > 0$, the procedure generates the sequence of iterates $\dots, -1, 1, -1, 1, \dots$, thus failing to converge.
\end{proof}

\subsection{Propositions \ref{prop1} and \ref{prop2}}\label{example:ps} 

We can prove Prop. \ref{prop1} and Prop. \ref{prop2} by giving two examples where the polarization effect happens or group-wise loss disparity increases.

\begin{example}[Polarization effects of $\boldsymbol{\theta}^{\text{PS}}$] \label{example:ps_retention}
    Consider individuals from two groups $a,b$ whose participation in an ML system depends on their perceived loss, i.e., group fraction $p_s^{(t)}$ changes based on the deployed ML model. Suppose individuals from group $a$ and $b$ have fixed data $Z = z_a = 1$ and $Z = z_b = 0$, respectively. Their initial group fractions are $p_a^{(0)} = 0.7, p_b^{(0)} = 0.3$. Consider conventional RRM with mean squared error $\mathcal{L}(\boldsymbol{\theta}; \mathcal{D}^{(t)})=\sum_{s\in\{a,b\}}p_s^{(t)}(\boldsymbol{\theta} - z_s)^2$ and transition $\text{T}$ satisfies $p_a^{(t+1)} = p_a^{(t)} + 0.1 \cdot \left[\mathcal{L}(\boldsymbol{\theta}^{(t)}; \mathcal{D}_b^{(t)})-\mathcal{L}(\boldsymbol{\theta}^{(t)}; \mathcal{D}_a^{(t)})\right], p_b^{(t+1)} = 1 - p_a^{(t+1)}$ (i.e., group fraction decreases if the group has higher loss). Next, consider $\mathcal{L}(\boldsymbol{\theta}; \mathcal{D}^{(t)})= p_a^{(t)}(\boldsymbol{\theta}-1)^2 + p_b^{(t)}\boldsymbol{\theta}^2$. Minimizing the loss results in $\boldsymbol{\theta}^{(t)} = \frac{p_a^{(t)}}{p_a^{(t)} + p_b^{(t)}}$. This means when $p_a^{(t)} > 0.5$, $\boldsymbol{\theta}^{(t)} > 0.5$ and $\mathcal{L}(\boldsymbol{\theta}^{(t)}; \mathcal{D}_b^{(t)}) > \mathcal{L}(\boldsymbol{\theta}^{(t)}; \mathcal{D}_a^{(t)})$, which further results in the increase of $p_a^{(t+1)}$. Suppose $\boldsymbol{\theta}^{(0)} = 0.7$ is the risk minimizer with respect to the initial distribution. Thus, $p_a^{(t)}$ (resp. $ p_b^{(t)}$) monotonically increases (resp. decreases) in $t$, and the PS solution solely contains people from group $a$ (see Fig. \ref{fig:e34}). 
\end{example}

\begin{example}[Exacerbated group-wise loss disparity]\label{example:ps_sc}
    Consider individuals from two groups $a,b$ who are subject to certain ML decisions; each has an initial feature $X\sim \text{Uniform}(0,1)$ and a binary label $Y = \mathbf{1}(X \ge 0.5)\in \{0,1\}$. Suppose the population has fixed group fractions $p_a = 0.7, p_b = 0.3$, but individuals may strategically manipulate their features to increase their chances of receiving favorable ML decisions, i.e., distribution of each group $\mathcal{D}_s^{(t)}$ changes based on ML model. Following the individual response model in \citep{perdomo_performative_2021}, we suppose individuals from two groups have certain effort budgets $\eta_a = 0.2, \eta_b = 0.01$ that can be used to manipulate their features $X_s^{(t+1)} = X_s + \eta_s \cdot \boldsymbol{\theta}^{(t)}$ without changing $Y$. Consider RRM iterative algorithm where the model parameter $\boldsymbol{\theta}^{(t)}$ is updated to maximize prediction accuracy  $\Pr(\mathbf{1}(X^{(t)} \ge \boldsymbol{\theta}^{(t)})= Y)$. The initial model $\boldsymbol{\theta}^{(0)} = 0.5$ which is the risk minimizer with respect to the initial distribution. Then at $t = 1$, all individuals from group $a$ has features increasing $0.1$ but their labels are unchanged. Because group $a$ is the majority group, the risk minimizer at next round $\boldsymbol{\theta}^{(1)}$ increases to $0.6$. $\boldsymbol{\theta}^{(t)}$ will then keep increasing over time until its convergence to PS solution $\boldsymbol{\theta}^{\text{PS}} = 0.625$ (see Fig. \ref{fig:e34}). However, at the PS solution, the accuracy for group $a$ is 1 but is only $0.88$ for group $b$. 
\end{example}
\subsection{Proof of Proposition \ref{prop:unique_fair_PS}}

We will first present lemmas for this proposition.

\begin{lemma} (First order optimality condition of convex functions \citep{Bubeck2015book}).

    If $f$ is convex and $\Omega$ is a closed convex set on which $f$ is differentiable, then
    \begin{equation}
        x^* \in \arg \min_{x \in \Omega} f(x),
    \end{equation}
    if and only if $(y - x^*)^\top \nabla g(x^*) \geq 0$ for all $y \in \Omega$.
\end{lemma}

\begin{lemma} \label{lemma:fair_loss_gamma_convex} ($\mathcal{L}_{\textsf{fair}}$ under fair penalty mechanisms is $\gamma$-strongly convex)

Under the conditions in Lemma \ref{lemma:SDPP}, the loss aware objective $\mathcal{L}_{\textsf{fair}}$ in Eqn \eqref{eq:squared_group_loss_penalty} and \eqref{eq:fair_loss_sample_level} are $\gamma$-strongly convex.
\end{lemma}

\begin{proof}
    We will show this for the sample level fair penalty Eqn \eqref{eq:fair_loss_sample_level} first. 
    
    From conditions in Lemma \ref{lemma:SDPP}, we know that $\ell$ is $\gamma$-strongly convex, and by definition $\ell \geq 0$, we will show that $\ell^2(\boldsymbol{\theta}; Z)$ is also convex in $\boldsymbol{\theta}$. For simplicity of notation, we will derive for an arbitrary fixed $Z$ and omit $Z$ in the following derivation and use $\ell(\boldsymbol{\theta})$.

    From the definition of convexity, we have $\forall \alpha \in (0,1), \forall \boldsymbol{\theta}_1, \boldsymbol{\theta}_2 \in \Theta$,
    \begin{equation}
        \ell((1-\alpha) \boldsymbol{\theta}_1 + \alpha \boldsymbol{\theta}_2) \leq (1 - \alpha) \ell(\boldsymbol{\theta}_1) + \alpha \ell(\boldsymbol{\theta}_2).
    \end{equation}
    Squaring both sides in the above gives us
    \begin{equation}
    \begin{aligned}
        \ell^2((1-\alpha) \boldsymbol{\theta}_1 + \alpha \boldsymbol{\theta}_2) \leq &~ (1 - \alpha)^2 \ell^2(\boldsymbol{\theta}_1) + \alpha^2 \ell^2(\boldsymbol{\theta}_2) + 2 \alpha (1 - \alpha) \ell(\boldsymbol{\theta}_1) \ell(\boldsymbol{\theta}_2) \\
        = &~ (1 - \alpha)^2 \ell^2(\boldsymbol{\theta}_1) + \alpha^2 \ell^2(\boldsymbol{\theta}_2) + 2 \alpha (1 - \alpha) \ell(\boldsymbol{\theta}_1) \ell(\boldsymbol{\theta}_2) \\
        &~ - (1 - \alpha) \ell^2(\boldsymbol{\theta}_1) - \alpha \ell^2(\boldsymbol{\theta}_2) + (1 - \alpha) \ell^2(\boldsymbol{\theta}_1) + \alpha \ell^2(\boldsymbol{\theta}_2) \\
        = &~ -\alpha(1-\alpha)[\ell(\boldsymbol{\theta}_1) - \ell(\boldsymbol{\theta}_2)]^2 + (1 - \alpha) \ell^2(\boldsymbol{\theta}_1) + \alpha \ell^2(\boldsymbol{\theta}_2) \\
        \leq &~ (1 - \alpha)\ell^2(\boldsymbol{\theta}_1) + \alpha \ell^2(\boldsymbol{\theta}_2)
    \end{aligned}
    \end{equation}
    which shows the convexity of $\ell^2$. 
    
    Therefore, $\ell(\boldsymbol{\theta}) + \rho \ell^2(\boldsymbol{\theta})$ is a sum of a $\gamma$-strongly convex function and a convex function and thus is $\gamma$-strongly convex. Then since $\mathcal{L}_{\textsf{fair}}$ is an affine combination of $\gamma$-strongly convex functions, it is also $\gamma$-strongly convex.

    For Eqn \eqref{eq:squared_group_loss_penalty}, since $\mathcal{L}_s$ is $\gamma$-strongly convex for $\forall s$, we can similarly show $\mathcal{L}_{\textsf{fair}}$ is $\gamma$-strongly convex.
\end{proof}

\begin{lemma} \label{lemma:fair_loss_tilde_beta_smooth} ($\mathcal{L}_{\textsf{fair}}$ under fair penalty mechanisms is $\Tilde{\beta}$-jointly smooth)

Under conditions in Lemma \ref{lemma:SDPP}, the loss aware objective $\mathcal{L}_{\textsf{fair}}$ in Eqn \eqref{eq:squared_group_loss_penalty} and \eqref{eq:fair_loss_sample_level} are $\Tilde{\beta}$-jointly smooth, where $\Tilde{\beta} = (2 \rho \overline{\ell} + 1) \beta + 2\rho\tilde{\ell}^2$.
\end{lemma}

\begin{proof}
    Denote $g := \ell^2$, then we have $\nabla_{\boldsymbol{\theta}} g = 2 \ell \cdot \nabla_{\boldsymbol{\theta}} \ell$, and thus we have
\begin{equation}
\begin{aligned}
&~ \|\nabla_{\boldsymbol{\theta}} g (\boldsymbol{\theta}; Z) - \nabla_{\boldsymbol{\theta}} g (\boldsymbol{\theta}'; Z)\|_2 
\;\leq\; 2 \beta \max\{ \|\ell (\boldsymbol{\theta}; Z)\|, \|\ell (\boldsymbol{\theta}'; Z)\| \} \| \boldsymbol{\theta} - \boldsymbol{\theta}'\|_2 
\;+\; 2 \,\tilde{\ell}^{\,2}\;\| \boldsymbol{\theta} - \boldsymbol{\theta}'\|_2, \\
&~ \|\nabla_{\boldsymbol{\theta}} g (\boldsymbol{\theta}; Z) - \nabla_{\boldsymbol{\theta}} g (\boldsymbol{\theta}; Z')\|_2 
\;\leq\;  2 \beta \max\{ \|\ell (\boldsymbol{\theta}; Z)\|, \|\ell (\boldsymbol{\theta}; Z') \|\} \| Z - Z'\|_2
\;+\;2 \,\tilde{\ell}^2\;\| Z - Z'\|_2 .
\end{aligned}
\end{equation}
    which implies $\left((2 \rho \overline{\ell} + 1) \beta + 2\rho\tilde{\ell}^2\right)$-joint smoothness of $\mathcal{L}_{\textsf{fair}}$ in Eqn \eqref{eq:squared_group_loss_penalty} and \eqref{eq:fair_loss_sample_level}.
\end{proof}

\begin{lemma} \label{lemma:fair_reweight_objective_smoothness}
    For fair re-weighting mechanisms,
    \begin{equation}
        \boldsymbol{\theta}^{(t+1)} = \arg \min_{\boldsymbol{\theta}} \mathcal{L}_{\textsf{fair}} = \arg\min_{\boldsymbol{\theta}} \sum_{s} p_s^{(t)} ( 1 + \rho \mathcal{L}_s(\boldsymbol{\theta}^{(t)}; \mathcal{D}^{(t-1)}_s) \mathcal{L}_s(\boldsymbol{\theta}); \mathcal{D}^{(t)}_s,
    \end{equation}
    where $\sum_{s} p_s^{(t)} ( 1 + \rho \mathcal{L}_s(\boldsymbol{\theta}^{(t)}; \mathcal{D}^{(t-1)}_s) \mathcal{L}_s(\boldsymbol{\theta}; \mathcal{D}^{(t)}_s)$ is $\gamma$-strongly convex and $(1+\rho \overline{\ell})$-jointly smooth.
\end{lemma}

\begin{proof}
    At a high level, $( 1 + \rho \mathcal{L}_s(\boldsymbol{\theta}^{(t)}; \mathcal{D}^{(t-1)}_s)$ acts like a scalar that may scale the strong convexity and joint smoothness coefficient.
    
    Since $\mathcal{L}_s(\boldsymbol{\theta}^{(t)}; \mathcal{D}^{(t-1)}_s) \geq 0$, we know the $\gamma$-strong convexity part holds.

    Similarly, since $\mathcal{L}_s(\boldsymbol{\theta}^{(t)}; \mathcal{D}^{(t-1)}_s) \leq \overline{\ell}$, we know the $(1+\rho \overline{\ell})$-joint smoothness holds.
\end{proof}

This lemma shows that we can treat $\mathcal{L}_s(\boldsymbol{\theta}^{(t)}; \mathcal{D}^{(t-1)}_s) \mathcal{L}_s(\boldsymbol{\theta}; \mathcal{D}^{(t)}_s)$ as the fair objective in fair er-weighting mechanisms, where this objective has very similar properties as the fair objectives using the fair penalty mechanism.

Based on the above lemmas, we present a result that bounds the fair-aware loss minimization solution distance from two different input distributions, where \citep{perdomo_performative_2021} and \citep{brown2020performative} derived similar bounds in the PP cases without fairness penalties. 

\begin{lemma} \label{lemma:minimizers_bound} (Bounding the fair objective minimizers)

    Given fair-aware loss function $\mathcal{L}_{\textsf{fair}}$ that is $\gamma$-strongly convex and $\Tilde{\beta}$-jointly smooth. Then for two distributions $\mathcal{D}, \Tilde{\mathcal{D}} \in \triangle(\mathcal{Z})$, denote
    \begin{equation}
        \boldsymbol{\theta}_{\textsf{fair}}^* = G(\mathcal{D}; \mathcal{L}_{\textsf{fair}}) := \arg \min_{\boldsymbol{\theta}} \mathcal{L}_{\textsf{fair}} (\boldsymbol{\theta}; \mathcal{D}), ~~~ \Tilde{\boldsymbol{\theta}}_{\textsf{fair}}^* = G(\Tilde{\mathcal{D}}; \mathcal{L}_{\textsf{fair}}) := \arg \min_{\boldsymbol{\theta}} \mathcal{L}_{\textsf{fair}} (\boldsymbol{\theta}; \Tilde{\mathcal{D}}),
    \end{equation}
    we have
    \begin{equation}
        \| \boldsymbol{\theta}_{\textsf{fair}}^* - \Tilde{\boldsymbol{\theta}}_{\textsf{fair}}^* \|_2 \leq \frac{\Tilde{\beta}}{\gamma} \mathcal{W}_1(\mathcal{D}, \Tilde{\mathcal{D}}).
    \end{equation}
\end{lemma}

\begin{proof}
    Using the strong convexity property and the fact that $G(\mathcal{D}; \mathcal{L}_{\textsf{fair}})$ is the unique minimizer of $\mathcal{L}_{\textsf{fair}}(\boldsymbol{\theta}; \mathcal{D})$, we can derive that 
    \begin{equation}
        - \gamma \| G(\mathcal{D}; \mathcal{L}_{\textsf{fair}}) - G(\Tilde{\mathcal{D}}; \mathcal{L}_{\textsf{fair}}) \|_2^2 \geq (G(\mathcal{D}; \mathcal{L}_{\textsf{fair}}) - G(\Tilde{\mathcal{D}}; \mathcal{L}_{\textsf{fair}}))^\top \nabla_{\boldsymbol{\theta}} \mathcal{L}_{\textsf{fair}}(\boldsymbol{\theta}; \mathcal{D}).
    \end{equation}

    Then for Eqn \eqref{eq:fair_loss_sample_level}, we observe that 
    \begin{equation}
        (G(\mathcal{D}; \mathcal{L}_{\textsf{fair}}) - G(\Tilde{\mathcal{D}}; \mathcal{L}_{\textsf{fair}}))^\top \nabla_{\boldsymbol{\theta}} \ell(\boldsymbol{\theta}; Z)
    \end{equation}
    is $ \| G(\mathcal{D}; \mathcal{L}_{\textsf{fair}}) - G(\Tilde{\mathcal{D}}; \mathcal{L}_{\textsf{fair}}) \|_2 \beta $-Lipschitz in $Z$, and \begin{equation}
        (G(\mathcal{D}; \mathcal{L}_{\textsf{fair}}) - G(\Tilde{\mathcal{D}}; \mathcal{L}_{\textsf{fair}}))^\top \nabla_{\boldsymbol{\theta}} \ell^2(\boldsymbol{\theta}; Z)
    \end{equation}
    is $2 \overline{\ell}  \| G(\mathcal{D}; \mathcal{L}_{\textsf{fair}}) - G(\Tilde{\mathcal{D}}; \mathcal{L}_{\textsf{fair}}) \|_2 \beta $-Lipschitz in $Z$, which follows from applying the Cauchy Schwartz inequality and the fact that $\ell$ is $\beta$-jointly smooth.

    Then we can derive that 
    \begin{equation}
        (G(\mathcal{D}; \mathcal{L}_{\textsf{fair}}) - G(\Tilde{\mathcal{D}}; \mathcal{L}_{\textsf{fair}}))^\top \nabla_{\boldsymbol{\theta}} \mathcal{L}_{\textsf{fair}}(\boldsymbol{\theta}; \mathcal{D}) \geq - (2 \rho \overline{\ell} + 1)\beta \mathcal{W}_1(\mathcal{D}, \Tilde{\mathcal{D}}),
    \end{equation}
    and thus
    \begin{equation}
        - \gamma \| G(\mathcal{D}; \mathcal{L}_{\textsf{fair}}) - G(\Tilde{\mathcal{D}}; \mathcal{L}_{\textsf{fair}}) \|_2^2 \geq - (2 \rho \overline{\ell} + 1)\beta \mathcal{W}_1(\mathcal{D}, \Tilde{\mathcal{D}}),
    \end{equation}
    we get the proof for Eqn \eqref{eq:fair_loss_sample_level}, and we can similarly prove this for Eqn \eqref{eq:squared_group_loss_penalty} and \eqref{eq:fair_reweight}.
\end{proof}

\begin{lemma} \label{lemma:fixed_point_ditribution_dist} (Bounding the fixed point distribution distance using parameter distance \citep{brown2020performative}) 

    Suppose the transition mapping $T$ is $\epsilon$-jointly sensitive with $\epsilon \in (0,1)$, then for $\boldsymbol{\theta}_1, \boldsymbol{\theta}_2$ and their corresponding fixed point distributions $\mathcal{D}_1, \mathcal{D}_2$, it holds that 
    \begin{equation}
        \mathcal{W}_1(\mathcal{D}_1, \mathcal{D}_2) \leq \frac{\epsilon}{1 - \epsilon} \| \boldsymbol{\theta}_1 - \boldsymbol{\theta}_2 \|_2
    \end{equation}
\end{lemma}

\textbf{Proposition \ref{prop:unique_fair_PS}.} (Unique Fair-PS solution)
Under Conditions in Lemma \ref{lemma:SDPP}, if $\epsilon(1 + \Tilde{\beta}/\gamma) < 1$, then for a given combination of (1) initial distribution $\mathcal{D}^{(0)}$, and (2) fair mechanism with strength $\rho$, there is a unique Fair-PS solution.

\begin{proof}
    We define the fixed point transition mapping as 
    \begin{equation}
        \text{T}^{\text{FP}}(\boldsymbol{\theta};\mathcal{D}) := \mathcal{D}_{\boldsymbol{\theta}} = \text{T}^{\text{FP}}(\boldsymbol{\theta})
    \end{equation}
    the mapping returning the model's fixed point distribution. 
    
    Then using the results in Lemma \ref{lemma:fair_loss_gamma_convex}, \ref{lemma:fair_loss_tilde_beta_smooth}, \ref{lemma:minimizers_bound}, 
    \ref{lemma:fair_reweight_objective_smoothness},
    \ref{lemma:fixed_point_ditribution_dist}, we can see that 
    \begin{equation}
        \boldsymbol{\theta}^{(t+1)} = G(\text{T}^{\text{FP}}(\boldsymbol{\theta}^{(t)}) ;\mathcal{L}_{\textsf{fair}})
    \end{equation}
    is a contraction mapping when conditions in Lemma \ref{lemma:SDPP} hold since $\epsilon(1 + \Tilde{\beta}/\gamma) < 1$. Therefore, using the Banach fixed-point theorem, we know that there is a unique Fair-PS pair. 
\end{proof}

\subsection{Proof of Theorem \ref{thm:fair_RRM}}

In this part, we provide the convergence of the class of Fair-RRM algorithms in Algorithm \ref{alg:fair_RRM}.

\begin{lemma} \label{lemma:delay_RRM_one_step_bound} (\citep{brown2020performative}) 
    Under conditions in Lemma \ref{lemma:SDPP}, given a decision model with parameter $\boldsymbol{\theta}$, denote the parameter returned by Delayed Deployment Scheme as $\hat{\mathcal{D}}_{\theta} := \text{T}^{DL}(\boldsymbol{\theta}; \mathcal{D})$ and the fixed point distribution of $\boldsymbol{\theta}$ as $\mathcal{D}_{\boldsymbol{\theta}}$, then 
    \begin{equation}
    \mathcal{W}_1(\hat{\mathcal{D}}_{\boldsymbol{\theta}}, \mathcal{D}_{\boldsymbol{\theta}}) \leq \frac{\epsilon}{1 - \epsilon} \delta
    \end{equation}
\end{lemma}

\begin{proof} \footnote{Due to different notations, we present the proof in \citep{brown2020performative} to help the readers. }
    For a fixed $\boldsymbol{\theta}$ and $\epsilon < 1$, the map $\text{T}(\boldsymbol{\theta} ; \cdot )$ is contracting with Lipschitz coefficient $\epsilon$ and has a unique fixed point $\mathcal{D}_{\boldsymbol{\theta}}$. Note that $\hat{\mathcal{D}}_{\boldsymbol{\theta}} = \text{T}^{\lceil r \rceil + 1}$, denote $\mathcal{D}^{1} = T (\boldsymbol{\theta}; \mathcal{D}^0)$, then 
    \begin{equation}
        \mathcal{W}_1(\hat{\mathcal{D}}_{\boldsymbol{\theta}}, \mathcal{D}_{\boldsymbol{\theta}}) \leq \frac{\epsilon^{\lceil r \rceil}}{1 - \epsilon} \mathcal{W}_1(\mathcal{D}^0, \mathcal{D}^1),
    \end{equation}
    then since $r =  \log^{-1}\left( \frac{1}{\epsilon} \right) \log\left( \frac{\mathcal{W}_1(\mathcal{D}^{0}, \mathcal{D}^{1})}{\delta} \right) $, we have $\epsilon^{\lceil r \rceil} < \epsilon^r = \frac{\delta}{\mathcal{W}_1(\mathcal{D}^{0}, \mathcal{D}^{1})}$, which completes the proof.
\end{proof}

\textbf{Theorem \ref{thm:fair_RRM}.}
(Fair-RRM Convergence)
    Under conditions in Lemma \ref{lemma:SDPP}:
    
        (i) If $\epsilon(1 + \Tilde{\beta}/\gamma) < 1$, Algorithm \ref{alg:fair_RRM} under the Conventional Schema converges to the Fair-PS pair at a linear rate.
        
        (ii) If $\epsilon(1 + \Tilde{\beta}/\gamma) < 1 - \epsilon$, Algorithm \ref{alg:fair_RRM} under the $k$-Delayed Schema converges to the Fair-PS pair at a linear rate for any $k$. 
        
        (iii) If $\epsilon(1 + \Tilde{\beta}/\gamma) < 1$, Algorithm \ref{alg:fair_RRM} under the Conventional Schema converges to a $\delta$ neighborhood of the Fair-PS pair in $O(\log^2 \frac{1}{\delta})$ steps.

\begin{proof}[proof of (i)]

    We will define a distance metric
    \begin{equation}
        d_{\text{pair}}( (\boldsymbol{\theta}_1, \mathcal{D}_1), (\boldsymbol{\theta}_2, \mathcal{D}_2) ) := \mathcal{W}_1(\mathcal{D}_1, \mathcal{D}_2) + \| \boldsymbol{\theta}_1 - \boldsymbol{\theta}_2 \|_2.
    \end{equation}

    Denote the Fair-RRM mapping as
    \begin{equation} \label{eq:Fair_RRM_mapping}
        F_{\textsf{fair}}(\boldsymbol{\theta}, \mathcal{D}) := (G_{\textsf{fair}}(\boldsymbol{\theta}, \mathcal{D}), \text{T}(\boldsymbol{\theta}; \mathcal{D})) 
    \end{equation}
    where $G_{\textsf{fair}}(\boldsymbol{\theta}, \mathcal{D}) := G(\text{T}(\boldsymbol{\theta}; \mathcal{D});\mathcal{L}_{\textsf{fair}})$.
    Then 
    \begin{equation}
        d_{\text{pair}}( F_{\textsf{fair}}(\boldsymbol{\theta}, \mathcal{D}), F_{\textsf{fair}}(\boldsymbol{\theta}',\mathcal{D}') ) = \mathcal{W}_1(\text{T}(\boldsymbol{\theta}; \mathcal{D}), \text{T}(\boldsymbol{\theta}'; \mathcal{D}')) + \| G_{\textsf{fair}}(\boldsymbol{\theta}, \mathcal{D}) - G_{\textsf{fair}}(\boldsymbol{\theta}', \mathcal{D}') \|_2.
    \end{equation}
    
    The $\epsilon$-jointly sensitivity of the transition map $\text{T}$ yields
    \begin{equation}
        \mathcal{W}_1(\text{T}(\boldsymbol{\theta}; \mathcal{D}), \text{T}(\boldsymbol{\theta}'; \mathcal{D}')) \leq \epsilon \mathcal{W}_1(\mathcal{D}, \mathcal{D}') + \epsilon \| \boldsymbol{\theta} - \boldsymbol{\theta}' \|_2,
    \end{equation}
    and using Lemma \ref{lemma:minimizers_bound}, we get 
    \begin{equation}
        \| G_{\textsf{fair}}(\boldsymbol{\theta}, \mathcal{D}) - G_{\textsf{fair}}(\boldsymbol{\theta}', \mathcal{D}') \|_2 \leq \frac{\Tilde{\beta}}{\gamma} 
        \mathcal{W}_1(\text{T}(\boldsymbol{\theta}; \mathcal{D}), \text{T}(\boldsymbol{\theta}'; \mathcal{D}')) \leq \epsilon \frac{\Tilde{\beta}}{\gamma} \mathcal{W}_1(\mathcal{D}, \mathcal{D}') + \epsilon \frac{\Tilde{\beta}}{\gamma} \| \boldsymbol{\theta} - \boldsymbol{\theta}' \|_2.
    \end{equation}
    Combining the above two equations, we get that under conditions in Lemma \ref{lemma:SDPP}, the Fair-RRM mapping is a contraction mapping that has a unique fixed point, where the fixed point satisfies the criteria of being the PS solution.

    The proof of this part referenced  \citep{brown2020performative} Theorem 4.
\end{proof}

\begin{proof}[proof of (ii)]

    We can derive the sensitivity of $T^k(\mathcal{D}, \boldsymbol{\theta})$ as follows
    \begin{equation}
        \begin{aligned}
            & \mathcal{W}_1(T^k(\mathcal{D}, \boldsymbol{\theta}), T^k(\mathcal{D}', \boldsymbol{\theta}')) \\
            \leq & \epsilon \mathcal{W}_1(T^{k-1}(\mathcal{D}, \boldsymbol{\theta}), T^{k-1}(\mathcal{D}', \boldsymbol{\theta}')) + \epsilon \| \boldsymbol{\theta} - \boldsymbol{\theta}' \|_2 \\
            \leq & \epsilon^2 \mathcal{W}_1(T^{k-2}(\mathcal{D}, \boldsymbol{\theta}), T^{k-2}(\mathcal{D}', \boldsymbol{\theta}')) + (\epsilon + \epsilon^2) \| \boldsymbol{\theta} - \boldsymbol{\theta}' \|_2 \\
            \leq & \dots \\
            < & \epsilon^k \mathcal{W}_1(\mathcal{D}, \mathcal{D}') + \frac{\epsilon}{1-\epsilon} \| \boldsymbol{\theta} - \boldsymbol{\theta}' \|_2 \\
            < & \frac{\epsilon}{1-\epsilon} \mathcal{W}_1(\mathcal{D}, \mathcal{D}') + \frac{\epsilon}{1-\epsilon} \| \boldsymbol{\theta} - \boldsymbol{\theta}' \|_2
        \end{aligned}
    \end{equation}
    So $T^k(\mathcal{D}, \boldsymbol{\theta})$ is $(\frac{\epsilon}{1-\epsilon})$-jointly sensitive.

    Using the proof of contraction mapping in part (i), we complete the proof of part (ii).
\end{proof}

\begin{proof}[proof of (iii)]
    
    Similar to the above, we denote the parameter returned by the Delayed Deployment Scheme as $\hat{\mathcal{D}}_{\boldsymbol{\theta}} = \text{T}^{DL}(\boldsymbol{\theta}; \cdot)$.
    We have $\boldsymbol{\theta}^{(t+1)} = G(\hat{\mathcal{D}}_{\boldsymbol{\theta}^{(t)}}; \mathcal{L}_{\textsf{fair}})$.
    
    Recall that $\boldsymbol{\theta}^{\text{PS}}_{\textsf{fair}} = G(\mathcal{D}^{\text{PS}}_{\textsf{fair}}; \mathcal{L}_{\textsf{fair}})$. Then Lemma \ref{lemma:minimizers_bound} indicates that 
    \begin{equation} \label{eq:thm_delay_rrm_p1}
        \| \boldsymbol{\theta}^{(t+1)} - \boldsymbol{\theta}^{\text{PS}}_{\textsf{fair}} \|_2 = \| G(\hat{\mathcal{D}}_{\boldsymbol{\theta}^{(t)}}; \mathcal{L}_{\textsf{fair}}) - G(\mathcal{D}^{\text{PS}}_{\textsf{fair}}; \mathcal{L}_{\textsf{fair}}) \|_2 \leq \frac{\Tilde{\beta}}{\gamma} \cdot \mathcal{W}_1 (\hat{\mathcal{D}}_{\boldsymbol{\theta}^{(t)}}, \mathcal{D}^{\text{PS}}_{\textsf{fair}}).
    \end{equation}
    Using the triangle inequality, we have
    \begin{equation}
        \mathcal{W}_1 (\hat{\mathcal{D}}_{\boldsymbol{\theta}}, \mathcal{D}^{\text{PS}}_{\textsf{fair}}) \leq \mathcal{W}_1 (\hat{\mathcal{D}}_{\boldsymbol{\theta}}, \mathcal{D}_{\boldsymbol{\theta}^{(t)}}) + \mathcal{W}_1 (\mathcal{D}_{\boldsymbol{\theta}^{(t)}}, \mathcal{D}^{\text{PS}}_{\textsf{fair}}),
    \end{equation}
    where $\mathcal{D}_{\boldsymbol{\theta}^{(t)}}$ is the fixed point distribution of $\boldsymbol{\theta}^{(t)}$.

    We can use Lemma \ref{lemma:delay_RRM_one_step_bound} to get $\mathcal{W}_1 (\hat{\mathcal{D}}_{\boldsymbol{\theta}}, \mathcal{D}_{\boldsymbol{\theta}^{(t)}}) \leq \frac{\epsilon \delta}{1-\epsilon}$ and use Lemma \ref{lemma:fixed_point_ditribution_dist} and the sensitivity definition to get
    $\mathcal{W}_1 (\mathcal{D}_{\boldsymbol{\theta}^{(t)}}, \mathcal{D}^{\text{PS}}_{\textsf{fair}}) \leq \frac{\epsilon}{1-\epsilon} \| \boldsymbol{\theta}^{(t)} - \boldsymbol{\theta}^{\text{PS}}_{\textsf{fair}} \|_2$.

    Therefore,
    \begin{equation}
        \mathcal{W}_1 (\hat{\mathcal{D}}_{\boldsymbol{\theta}}, \mathcal{D}^{\text{PS}}_{\textsf{fair}}) \leq \frac{\epsilon}{1-\epsilon}(\delta + \| \boldsymbol{\theta}^{(t)} - \boldsymbol{\theta}^{\text{PS}}_{\textsf{fair}} \|_2).
    \end{equation}

    When $\| \boldsymbol{\theta}^{(t)} - \boldsymbol{\theta}^{\text{PS}}_{\textsf{fair}} \|_2 > \delta$, we have 
    \begin{equation}
        \| \boldsymbol{\theta}^{(t+1)} - \boldsymbol{\theta}^{\text{PS}}_{\textsf{fair}} \|_2 < \frac{2 \epsilon}{1 - \epsilon} \frac{\Tilde{\beta}}{\gamma} \| \boldsymbol{\theta}^{(t)} - \boldsymbol{\theta}^{\text{PS}}_{\textsf{fair}} \|_2 \leq \| \boldsymbol{\theta}^{(t)} - \boldsymbol{\theta}^{\text{PS}}_{\textsf{fair}} \|_2.
    \end{equation}
    On the other hand, when $\| \boldsymbol{\theta}^{(t)} - \boldsymbol{\theta}^{\text{PS}}_{\textsf{fair}} \|_2 \leq \delta$, we have
    \begin{equation}
        \| \boldsymbol{\theta}^{(t+1)} - \boldsymbol{\theta}^{\text{PS}}_{\textsf{fair}} \|_2 \leq \frac{2 \epsilon}{1 - \epsilon} \frac{\Tilde{\beta}}{\gamma} \delta \leq \delta
    \end{equation}

    Combining the two cases together, we know that for $t \geq \left( 1 - \frac{2 \epsilon \Tilde{\beta}}{\gamma ( 1 - \epsilon)} \right)^{-1} \log\left( \frac{\boldsymbol{\theta}^{(0)} - \boldsymbol{\theta}^{\text{PS}}_{\textsf{fair}}}{\delta} \right)$, we have
    \begin{equation}
        \| \boldsymbol{\theta}^{(t)} - \boldsymbol{\theta}^{\text{PS}}_{\textsf{fair}} \|_2 \leq \left( \frac{2 \epsilon}{1 - \epsilon} \frac{\Tilde{\beta}}{\gamma} \right)^t \delta \leq \delta,
    \end{equation}
    which completes the proof for part (iii).

    The proof of this part referenced \citep{brown2020performative} Theorem 8.
\end{proof}

\subsection{Proof of Theorem \ref{thm:fair_RERM}}

In this part, we provide the convergence of the class of fair-RERM algorithms in Algorithm \ref{alg:fair_RERM}. 

\textbf{Theorem \ref{thm:fair_RERM}} (Fair RERM Convergence) 
    Under conditions in Lemma \ref{lemma:SDPP}, suppose $\exists \alpha > 1, \mu > 0$ such that $\int_{\mathbb{R}^m} e^{\mu |x|^{\alpha}} Z dx$ is finite $\forall Z \in \triangle(\mathcal{Z})$. 
    
    For given a convergence radius $\delta \in (0,1)$, take 
    $n_t = O\left(\frac{\log(t/p)}{(\epsilon(1+\Tilde{\beta}/\gamma)\delta)^m}\right)$ samples at $t$. If $2 \epsilon(1 + \Tilde{\beta}/\gamma) < 1$, then with probability $1-p$, the iterates of fair-RERM are within a radius $\delta$ of the Fair-PS pair for $t \geq (1 - 2 \epsilon(1 + \Tilde{\beta}/\gamma) O(\log (1/\delta))$.

\begin{proof}
    Given $n$ samples from $\mathcal{D}$, denote $\hat{\mathcal{D}}^{n}$ the empirical distribution obtained from them. 
    Denote 
    \begin{equation}
        \hat{G}^n_{\textsf{fair}}(\boldsymbol{\theta}, \mathcal{D}) := G(\text{T}(\boldsymbol{\theta}; \hat{\mathcal{D}}^n),
    \end{equation}
    (recall $G_{\textsf{fair}}(\boldsymbol{\theta}, \mathcal{D}) := G(\text{T}(\boldsymbol{\theta}; \mathcal{D});\mathcal{L}_{\textsf{fair}})$ in Eqn \eqref{eq:Fair_RRM_mapping}). And also denote the Fair-RERM mapping as
    \begin{equation} \label{eq:Fair_RERM_mapping}
        \hat{F}_{\textsf{fair}}(\boldsymbol{\theta}, \mathcal{D}) := (\hat{G}^n_{\textsf{fair}}(\boldsymbol{\theta}, \mathcal{D}), \text{T}(\boldsymbol{\theta}; \mathcal{D})) 
    \end{equation}

    By Theorem 2 of \citep{fournier_hal-00915365}, since $\exists \alpha > 1, \mu > 0$ such that $\int_{\mathbb{R}^m} e^{\mu |x|^{\alpha}} Z dx$ is finite $\forall Z \in \triangle(\mathcal{Z})$, then if $n_t = \mathcal{O}\left( \frac{1}{\epsilon(1+\Tilde{\beta}/\gamma)^m} \log(t/p) \right)$, where $m$ is the dimension of the sample, we have $\mathcal{W}_1(\mathcal{D}, \hat{\mathcal{D}}^{n}) \geq \epsilon(1 + \Tilde{\beta}/\gamma)\delta$ with probability at most $\frac{6p}{\pi^2 t^2}$. 
    
    Therefore,
    \begin{equation}
        \mathbb{P} \big(\mathcal{W}_1(\mathcal{D}, \hat{\mathcal{D}}^{n}) \leq \epsilon(1 + \Tilde{\beta}/\gamma)\delta \big), \forall t) = 1 - \sum_{t=1}^{\infty} \mathbb{P}\big(\mathcal{W}_1(\mathcal{D}, \hat{\mathcal{D}}^{n_t}) > \epsilon(1 + \Tilde{\beta}/\gamma)\delta \big) \geq 1 - \sum_{t=1}^{\infty} \frac{6p}{\pi^2 t^2} = 1-p
    \end{equation}
    i.e., with probability at least $1-p$ we have that for each time step $t$, it holds that 
    \begin{equation} \label{eq:equation_1}
        \mathcal{W}_1(\mathcal{D}, \hat{\mathcal{D}}^{n}) \leq \epsilon(1 + \Tilde{\beta}/\gamma)\delta.
    \end{equation}

    Then if $d_{\text{pair}}\big(( \boldsymbol{\theta}, \mathcal{D}), ( \boldsymbol{\theta}^{\text{PS}}_{\textsf{fair}}, \mathcal{D}^{\text{PS}}_{\textsf{fair}})\big) \geq \delta$, and Eqn \eqref{eq:equation_1} holds, we have 
    \begin{align}
        &~ d_{\text{pair}}( \hat{F}_{\textsf{fair}}(\boldsymbol{\theta}, \mathcal{D}), (\boldsymbol{\theta}^{\text{PS}}_{\textsf{fair}}, \mathcal{D}^{\text{PS}}_{\textsf{fair}}) ) \nonumber \\
        = &~ \mathcal{W}_1 (\text{T}(\boldsymbol{\theta}; \mathcal{D}), \mathcal{D}^{\text{PS}}_{\textsf{fair}}) + \| \hat{G}^n_{\textsf{fair}}(\boldsymbol{\theta}, \mathcal{D}) - \boldsymbol{\theta}^{\text{PS}}_{\textsf{fair}} \|_2 \nonumber \\
        \leq &~ \mathcal{W}_1 (\text{T}(\boldsymbol{\theta}; \mathcal{D}), \text{T}(  \boldsymbol{\theta}^{\text{PS}}_{\textsf{fair}}; \mathcal{D}^{\text{PS}}_{\textsf{fair}})) + \| \hat{G}^n_{\textsf{fair}}(\boldsymbol{\theta}, \mathcal{D}) - G_{\textsf{fair}}(\boldsymbol{\theta}, \mathcal{D}) \|_2 + \| G_{\textsf{fair}}(\boldsymbol{\theta}, \mathcal{D}) - G_{\textsf{fair}}(\boldsymbol{\theta}^{\text{PS}}_{\textsf{fair}}; \mathcal{D}^{\text{PS}}_{\textsf{fair}})\|_2  \nonumber\\
        \leq &~ \epsilon \mathcal{W}_1(\mathcal{D}, \mathcal{D}^{\text{PS}}_{\textsf{fair}}) + \epsilon \|\boldsymbol{\theta} - \boldsymbol{\theta}^{\text{PS}}_{\textsf{fair}} \|_2 + \frac{\Tilde{\beta}}{\gamma} \mathcal{W}_1 (\hat{\text{T}}^n(\boldsymbol{\theta}; \mathcal{D}), \text{T}(\boldsymbol{\theta}; \mathcal{D})) + \frac{\Tilde{\beta}}{\gamma} ( \text{T}(\boldsymbol{\theta}; \mathcal{D})), \text{T}(\boldsymbol{\theta}^{\text{PS}}_{\textsf{fair}}; \mathcal{D}^{\text{PS}}_{\textsf{fair}})) ) \nonumber\\
        \leq &~ \epsilon \mathcal{W}_1(\mathcal{D}, \mathcal{D}^{\text{PS}}_{\textsf{fair}}) + \epsilon \|\boldsymbol{\theta} - \boldsymbol{\theta}^{\text{PS}}_{\textsf{fair}} \|_2 + \frac{\Tilde{\beta}}{\gamma} \epsilon \bigg(1 + \frac{\Tilde{\beta}}{\gamma} \bigg) \delta + \frac{\Tilde{\beta}}{\gamma}\big( \epsilon \mathcal{W}_1(\mathcal{D}, \mathcal{D}^{\text{PS}}_{\textsf{fair}}) + \epsilon \|\boldsymbol{\theta} - \boldsymbol{\theta}^{\text{PS}}_{\textsf{fair}} \|_2\big) \nonumber \\
        = &~ \bigg(1 + \frac{\Tilde{\beta}}{\gamma} \bigg) d_{\text{pair}}\big(( \boldsymbol{\theta}, \mathcal{D}), ( \boldsymbol{\theta}^{\text{PS}}_{\textsf{fair}}, \mathcal{D}^{\text{PS}}_{\textsf{fair}})\big)  + \epsilon \bigg(1 + \frac{\Tilde{\beta}}{\gamma}\delta \bigg) \nonumber \\
        \leq &~ 2 \epsilon \bigg(1 + \frac{\Tilde{\beta}}{\gamma} \bigg) d_{\text{pair}}\big(( \boldsymbol{\theta}, \mathcal{D}), ( \boldsymbol{\theta}^{\text{PS}}_{\textsf{fair}}, \mathcal{D}^{\text{PS}}_{\textsf{fair}}) \big)
    \end{align}
    where the third line uses triangle inequality, fourth line uses Lemma \ref{lemma:minimizers_bound}, and fifth line uses Eqn \eqref{eq:equation_1}. In other words, the above shows that if the current pair is more than $\delta$ away from the Fair PS solution, then as long as Eqn \eqref{eq:equation_1} is true, the Fair RERM mapping is a contraction.

    Similarly, if $d_{\text{pair}}\big(( \boldsymbol{\theta}, \mathcal{D}, ( \boldsymbol{\theta}^{\text{PS}}_{\textsf{fair}}, \mathcal{D}^{\text{PS}}_{\textsf{fair}})\big) < \delta$, and Eqn \eqref{eq:equation_1} holds, we can show that     
    \begin{align}
         d_{\text{pair}}( \hat{F}_{\textsf{fair}}(\boldsymbol{\theta}, \mathcal{D}), (\boldsymbol{\theta}^{\text{PS}}_{\textsf{fair}}, \mathcal{D}^{\text{PS}}_{\textsf{fair}}) ) \leq \bigg(1 + \frac{\Tilde{\beta}}{\gamma} \bigg) d_{\text{pair}}\big(( \boldsymbol{\theta}, \mathcal{D}), ( \boldsymbol{\theta}^{\text{PS}}_{\textsf{fair}}, \mathcal{D}^{\text{PS}}_{\textsf{fair}})\big)  + \epsilon \bigg(1 + \frac{\Tilde{\beta}}{\gamma}\delta \bigg)
         \leq 2 \epsilon \bigg(1 + \frac{\Tilde{\beta}}{\gamma} \bigg) \delta < \delta,
    \end{align}
    which means the pair will stay in the $\delta$ radius ball once an iterate goes in the ball.

    Then we move on to show when $t \geq \bigg(1 - 2 \epsilon \bigg(1 + \frac{\Tilde{\beta}}{\gamma} \bigg) \log\bigg(\frac{d_{\text{pair}}\big(( \boldsymbol{\theta}^{(1)}, \mathcal{D}^{(0)}), ( \boldsymbol{\theta}^{\text{PS}}_{\textsf{fair}}, \mathcal{D}^{\text{PS}}_{\textsf{fair}}) \big) }{\delta}\bigg)$ the fair RERM mapping hits the $\delta$ radius ball centered at the Fair PS solution. Again, if Eqn \eqref{eq:equation_1} holds, then denote the initial pair as $(\boldsymbol{\theta}^{(1)}, \mathcal{D}^{(0)})$, then we can use similar arguments as above to show that 
    \begin{align}
        d_{\text{pair}}( (\boldsymbol{\theta}^{(t+1)}, \mathcal{D}^{(t)}), (\boldsymbol{\theta}^{\text{PS}}_{\textsf{fair}}, \mathcal{D}^{\text{PS}}_{\textsf{fair}}) ) \leq &~ 2 \epsilon \bigg(1 + \frac{\Tilde{\beta}}{\gamma} \bigg)^t d_{\text{pair}}\big(( \boldsymbol{\theta}^{(1)}, \mathcal{D}^{(0)}), ( \boldsymbol{\theta}^{\text{PS}}_{\textsf{fair}}, \mathcal{D}^{\text{PS}}_{\textsf{fair}}) \big) \nonumber \\
        \leq &~ exp\bigg(-t \bigg(1 - 2 \epsilon \bigg(1 + \frac{\Tilde{\beta}}{\gamma} \bigg)\bigg) d_{\text{pair}}\big(( \boldsymbol{\theta}^{(1)}, \mathcal{D}^{(0)}), ( \boldsymbol{\theta}^{\text{PS}}_{\textsf{fair}}, \mathcal{D}^{\text{PS}}_{\textsf{fair}}) \big) \nonumber \\
        \leq &~ \delta.
    \end{align}
    This completes the proof.
    
    For the $k$-Delayed RERM, we can follow the proof of part (ii) in Theorem \ref{thm:fair_RRM} that the transition $\text{T}^k$ is $\Tilde{\epsilon} := \frac{\epsilon}{1-\epsilon}$ jointly sensitive, and if we replace the $\epsilon$ terms with $\Tilde{\epsilon}$ and times a $k$ scalar in the number of iterations, we can get the similar result.
    
    The proof of this part referenced Theorem 5 in \citep{brown2020performative}
\end{proof}

\subsection{Proof of the fairness guarantee}\label{app:proof_fimp}

\textbf{Discussion of the intricate nature of proving the effectiveness of $\rho$ under the general PP setting.} We first note that it is non-trivial to find out whether fairness at PS solution increases as $\rho$ increases under the general setting where both the group fractions and the group-wise feature-label distributions change. Consider an example where the agents in the minority group are influenced by the majority group, i.e., when their group fraction becomes less than a threshold, their tastes are shaped by the mainstream culture. Then they can change their feature-label distributions to be more similar to the one of the majority group. This will possibly make the loss disparity smaller even when $\rho$ becomes larger. However, in majority of previous work on the retention of recommendation system \citep{Zhang_2019_Retention, duchi2018learning}, the group-wise distributions are assumed to be static. 

\textbf{Proof of Theorem \ref{thm:fair_improvement}.} With Assumption \ref{assumption:retention}, we state Lemma \ref{lemma:fair_single}.

\begin{lemma}[fairness guarantee at a single step]\label{lemma:fair_single}
    At each single step, the following results hold:  
    
    (i) Given $\mathcal{D}^{(t)}$ fixed (i.e., $p_a^{(t)}, p_b^{(t)}$ stay fixed), increasing $\rho$ in Eqn. \eqref{eq:squared_group_loss_penalty} leads to lower or equal $\triangle^{t}_{\textsf{fair},\mathcal{L}}(\rho)$ ; 
    
    (ii) Given the optimization objective fixed (i.e., $\rho$ stays fixed), lower fraction disparity $\triangle_p^{(t)}$ leads to lower or equal $\triangle^{t}_{\textsf{fair},\mathcal{L}}(\rho)$ when the minority group incurs higher loss.
\end{lemma}

\begin{proof}[Proof of (i)]

 We prove (i) by contradiction. 
 
 Assume $\rho_1 < \rho_2$, and denote the the optimized model parameters as $\theta_{\rho_1}, \theta_{\rho_2}$. Since $\mathcal{D}_s$ does not change over time, we can know the group-wise loss under $\rho_1, \rho_2$ at both rounds $L_{s,\rho_1}^{t} = L_{s, \rho_1}^{t-1} = L_{s, \rho_1}$ $L_{s,\rho_2}^{t} = L_{s, \rho_2}^{t-1} = L_{s, \rho_2}$. Thus, we can interchange the superscript arbitrarily (i.e., if the deployed models are the same, $L_s^{t} = L_s^{t-1}$). Wlog, assume $L_{a,\rho_1}^t > L_{b,\rho_1}^t$. Then consider all possible scenarios where the loss disparity may be larger:
    \begin{itemize}[leftmargin=*]
    \small
        \item $L_{a,\rho_1}^t < L_{a,\rho_2}^t$ and $L_{b,\rho_1}^t < L_{b,\rho_2}^t$: if this holds, then we can just shift $\theta_{\rho_2}$ to $\theta_{\rho_1}$ to guarantee smaller $L_{a,\rho_2}^t, L_{b,\rho_2}^t$ and therefore smaller $L_{a,\rho_2}^{t-1}, L_{b,\rho_2}^{t-1}$, resulting in a smaller loss specified by Eqn. \eqref{eq:squared_group_loss_penalty}. This produces a contradiction since $\theta_{\rho_2}$ is the minimizer;
        
        \item $L_{a,\rho_1}^t > L_{a,\rho_2}^t$ and $L_{b,\rho_1}^t > L_{b,\rho_2}^t$: In the same way, we can shift $\theta_{\rho_1}$ to $\theta_{\rho_2}$ to trivially decrease both $L_{a,\rho_1}^t, L_{b,\rho_1}^t$, violating the condition that $\theta_{\rho_1}$ is the minimizer;
        
        \item $L_{a,\rho_1}^t < L_{a,\rho_2}^t$ and $L_{b,\rho_1}^t > L_{b,\rho_2}^t$. According to the optimality at $\rho_1, \rho_2$, we have: 
        \begin{equation}\label{eq:opt_a}
            p_a \cdot \left (L_{a,\rho_2}^t - L_{a,\rho_1}^t  \right) + p_b \cdot \left (L_{b,\rho_2}^t - L_{b,\rho_1}^t  \right) + p_a \cdot \rho_2 \cdot \left ((L_{a,\rho_2}^t)^2 - (L_{a,\rho_1}^t)^2  \right) + p_b \cdot \rho_2 \cdot \left ((L_{b,\rho_2}^t)^2 - (L_{b,\rho_1}^t)^2  \right) < 0
        \end{equation}
        \begin{equation}\label{eq:opt_b}
            p_a \cdot \left (L_{a,\rho_2}^t - L_{a,\rho_1}^t  \right) + p_b \cdot \left (L_{b,\rho_2}^t - L_{b,\rho_1}^t  \right) + p_a \cdot \rho_1 \cdot \left ((L_{a,\rho_2}^t)^2 - (L_{a,\rho_1}^t)^2  \right) + p_b \cdot \rho_1 \cdot \left ((L_{b,\rho_2}^t)^2 - (L_{b,\rho_1}^t)^2  \right) > 0
        \end{equation}
      Subtract Eqn. \eqref{eq:opt_a} from Eqn. \eqref{eq:opt_b}, we get
      \begin{align}\label{eq:squareeq}
      &p_a \cdot  \left ( (L_{a,\rho_2}^t)^2 - (L_{a,\rho_1}^t)^2  \right) + p_b \cdot \left ((L_{b,\rho_2}^t)^2 - (L_{b,\rho_1}^t)^2  \right) < 0 \\
      \nonumber \Leftrightarrow ~ &p_a \cdot  \left(L_{a,\rho_2}^t + L_{a,\rho_1}^t  \right) \cdot \left(L_{a,\rho_2}^t - L_{a,\rho_1}^t \right)  + p_b \cdot \left (L_{b,\rho_2}^t + L_{b,\rho_1}^t \right) \cdot \left (L_{b,\rho_2}^t - L_{b,\rho_1}^t \right) < 0
      \end{align}
      From Eqn. \eqref{eq:squareeq} and Eqn. \eqref{eq:opt_b}, we can immediately get:
      \begin{equation}\label{eq:lineareq}
          p_a \cdot \left(L_{a,\rho_2}^t - L_{a,\rho_1}^t \right)  + p_b \cdot \left (L_{b,\rho_2}^t - L_{b,\rho_1}^t \right) > 0
      \end{equation}
      Otherwise Eqn. \eqref{eq:opt_b} would be smaller than $0$. We have by assumption $L_{a, \rho_2}^t> L_{a, \rho_1}^t > L_{b, \rho_1}^t > L_{b, \rho_2}^t$. Then $L_{a, \rho_1}^t + L_{a, \rho_2}^t > L_{b, \rho_1}^t + L_{b, \rho_2}^t$. Next, noticing that $L_{b,\rho_2}^t - L_{b,\rho_1}^t < 0$ and $L_{a,\rho_2}^t - L_{a,\rho_1}^t > 0$, then if we divide $L_{a, \rho_1}^t + L_{a, \rho_2}^t$ at both sides of Eqn. \eqref{eq:squareeq}, then we will get:
      \begin{equation}\label{eq:contradict}
          p_a \cdot  \left(L_{a,\rho_2}^t - L_{a,\rho_1}^t \right)  + p_b \cdot \frac{\left (L_{b,\rho_2}^t + L_{b,\rho_1}^t \right)}{\left(L_{a,\rho_2}^t + L_{a,\rho_1}^t  \right)} \cdot \left (L_{b,\rho_2}^t - L_{b,\rho_1}^t \right) < 0
      \end{equation}
      However, the LHS of Eqn. \eqref{eq:contradict} is larger than the LHS of Eqn. \eqref{eq:lineareq}, while the LHS requires the opposite, producing a contradiction. Finally, note that it is impossible to let loss of one group stays fixed while the other one changes where we can easily prove with similar contradictions.
      \normalsize
    \end{itemize}

    Since all the above scenarios do not hold, we have proved (i). Moreover, if for any $\rho_2 > \rho_1$, the loss disparities are equal, it means that both groups achieve minimal loss with a same $\theta^{*}$, which is very uncommon considering the difference of demographic groups. Thus, increasing $\rho$ commonly leads to lower loss disparity $\Delta_\mathcal{L}^{(t)}$.
\end{proof}

\vspace{0.2cm}

\begin{proof}[proof of (ii)]
    Next, we prove (ii) similarly. Denote the optimized group-wise loss under the two group fractions as $L_s^{(t)}$ and $L_{s'}^{(t)}$.  Wlog, assume $L_a^{(t)} > L_{b}^{(t)}$. We will have the following equations:

    \begin{equation}\label{eq:prate_eq}
        p_{a'}^{(t)} \cdot \left (L_{a'}^{(t)} - L_{a}^{(t)} \right) + p_{b'}^{(t)} \cdot \left (L_{b'}^{(t)} - L_{b}^{(t)} \right) < 0
    \end{equation}

    \begin{equation}\label{eq:prate_eq2}
        p_{a}^{(t)} \cdot \left (L_{a'}^{(t)} - L_{a}^{(t)} \right) + p_{b}^{(t)} \cdot \left (L_{b'}^{(t)} - L_{b}^{(t)} \right) > 0
    \end{equation}

    To prove by contradiction, we need to consider the situations where the group fraction discrepancy becomes smaller. This can be either $p_{a}^{(t)} < p_b^{(t)}, p_{a'}^{(t)} > p_{a}^{(t)}$ or $p_{a}^{(t)} > p_b^{(t)}, p_{a'}^{(t)} < p_{a}^{(t)}$. Then consider all possible scenarios where the loss disparity may be larger:
    \begin{itemize}
        \item $L_{a}^{(t)}  < L_{a'}^{(t)}$ and $L_{b}^{(t)}  < L_{b'}^{(t)}$: Eqn. \eqref{eq:prate_eq} does not hold;
        \item $L_{a}^{(t)}  > L_{a'}^{(t)}$ and $L_{b}^{(t)}  > L_{b'}^{(t)}$: Eqn. \eqref{eq:prate_eq2} does not hold;
        \item $L_{a}^{(t)}  < L_{a'}^{(t)}$ and $L_{b}^{(t)}  > L_{b'}^{(t)}$: Subtract Eqn. \eqref{eq:prate_eq} from Eqn. \eqref{eq:prate_eq2}, we can get:
        \begin{equation}\label{eq:new_prate}
            \left (p_{a}^{(t)} -  p_{a'}^{(t)}\right) \cdot \left (L_{a'}^{(t)} - L_{a}^{(t)} \right) + \left(p_{b}^{(t)} -  p_{b'}^{(t)}\right) \cdot \left (L_{b'}^{(t)} - L_{b}^{(t)} \right) > 0
        \end{equation}
    Noticing that $p_{a}^{(t)} -  p_{a'}^{(t)} = - \left( p_{b}^{(t)} -  p_{b'}^{(t)}\right)$, then only when $p_{a}^{(t)} -  p_{a'}^{(t)} > 0$ we can make the above inequality holds. According to the situations where the group fraction discrepancy becomes smaller, we know $p_{a}^{(t)} > p_{b}^{(t)}$ must hold and $a$ is the majority group. However, this contradicts the initial condition $L_a^{(t)} > L_{b}^{(t)}$, i.e, the majority group should not have higher group-wise loss. 
    \end{itemize}

Finally, it is easy to see the cases where losses stay fixed cannot satisfy Eqn. \eqref{eq:prate_eq} and Eqn. \eqref{eq:prate_eq2} simultaneously. Thus, we prove (ii).
\end{proof}

With Lemma \ref{lemma:fair_single}, we can easily prove Thm. \ref{thm:fair_improvement}. What we need to prove is $\triangle^{\text{PS}}_{\textsf{fair},\mathcal{L}}(\rho_1) \ge \triangle^{\text{PS}}_{\textsf{fair},\mathcal{L}}(\rho_2)$. Now we know at $t$, higher $\rho$ results in lower or equal $\triangle^{t}_{\textsf{fair},\mathcal{L}}(\rho)$ ((i) of Lemma \ref{lemma:fair_single}). Then the lower or equal $\triangle^{t}_{\textsf{fair},\mathcal{L}}(\rho)$ results in lower group fraction disparity at $t+1$ (Assumption \ref{assumption:retention}), and further causes lower loss disparity at $t+1$ ((ii) of Lemma \ref{lemma:fair_single}). This enables a forward induction to prove that $\triangle^{t}_{\textsf{fair},\mathcal{L}}(\rho)$ is non-decreasing with $\rho$ at each time step to infinity, which completes the whole proof of Thm. \ref{thm:fair_improvement}.

\subsection{Proof of Theorem \ref{thm:quantified_fairness_improvement}}

\textbf{Theorem \ref{thm:quantified_fairness_improvement}} 
Denote $p_s^{\text{PS}}$ as the fraction of group $s$ at $\mathcal{D}^{\text{PS}}$ under retention dynamics.
We assume \(\mathcal{L}(\boldsymbol{\theta}; \mathcal{D})\) is twice continuously differentiable. For sufficiently small \(\rho > 0\), we know $\Delta _{\mathcal{L}}^{\text{PS}} - \Delta _{\mathcal{L}, \text{fair}}^{\text{PS}}(\rho)
= 2\rho \, p_a^{\text{PS}} p_b^{\text{PS}} \, \Delta _{\mathcal{L}}^{\text{PS}} \cdot 
\boldsymbol{v}_{\text{fair}}^\top H^{-1} \boldsymbol{v}_{\text{fair}} + \mathcal{O}(\rho^2)$
where
\(
\boldsymbol{v}_{\text{fair}} = \nabla_{\boldsymbol{\theta}} \mathcal{L}(\boldsymbol{\theta}^{\text{PS}}; \mathcal{D}_a)
- \nabla_{\boldsymbol{\theta}} \mathcal{L}(\boldsymbol{\theta}^{\text{PS}}; \mathcal{D}_b),
H = \nabla_{\boldsymbol{\theta}}^2 \mathcal{L}(\boldsymbol{\theta}^{\text{PS}}; \mathcal{D}^{\text{PS}}).
\)

\begin{proof}
Under fixed  $\mathcal{D}_s$, the Fair-PS solution \(\boldsymbol{\theta}_\rho := \boldsymbol{\theta}^{\text{PS}}_{\text{fair}}(\rho)\) minimizes the fixed objective:
\[
\mathcal{L}_{\text{fair}}(\boldsymbol{\theta}, \rho) = \mathcal{L}(\boldsymbol{\theta}; \mathcal{D}^{\text{PS}}) 
+ \rho \sum_{s \in \{a, b\}} p_s^{\text{PS}} \left[ \mathcal{L}(\boldsymbol{\theta}; \mathcal{D}_s) \right]^2.
\]
Define \(F(\boldsymbol{\theta}, \rho) := \nabla_{\boldsymbol{\theta}} \mathcal{L}_{\text{fair}}(\boldsymbol{\theta}, \rho)\), so \(F(\boldsymbol{\theta}_\rho, \rho) = 0\) and \(\boldsymbol{\theta}_0 = \boldsymbol{\theta}_{\text{PS}}\).

By the implicit function theorem:
\[
\frac{d\boldsymbol{\theta}_\rho}{d\rho} = -\left( \nabla_{\boldsymbol{\theta}}^2 \mathcal{L}_{\text{fair}}(\boldsymbol{\theta}_\rho, \rho) \right)^{-1} 
\cdot \nabla_{\boldsymbol{\theta}} P(\boldsymbol{\theta}_\rho),
\]
where the penalty term is:
\[
P(\boldsymbol{\theta}) := \sum_{s \in \{a, b\}} p_s^{\text{PS}} \left[ \mathcal{L}(\boldsymbol{\theta}; \mathcal{D}_s) \right]^2.
\]
At \(\rho = 0\), this gives:
\[
\left. \frac{d\boldsymbol{\theta}_\rho}{d\rho} \right|_{\rho = 0} = -H^{-1} \nabla_{\boldsymbol{\theta}} P(\boldsymbol{\theta}^{\text{PS}}).
\]

Now compute the gradient of the penalty:
\[
\nabla_{\boldsymbol{\theta}} P(\boldsymbol{\theta}^{\text{PS}}) 
= 2 \sum_{s \in \{a, b\}} p_s^{\text{PS}} \mathcal{L}(\boldsymbol{\theta}^{\text{PS}}; \mathcal{D}_s) 
\nabla_{\boldsymbol{\theta}} \mathcal{L}(\boldsymbol{\theta}^{\text{PS}}; \mathcal{D}_s).
\]
Using the first-order condition for \(\boldsymbol{\theta}^{\text{PS}}\),
\[
p_a^{\text{PS}} \nabla_{\boldsymbol{\theta}} \mathcal{L}(\boldsymbol{\theta}^{\text{PS}}; \mathcal{D}_a)
+ p_b^{\text{PS}} \nabla_{\boldsymbol{\theta}} \mathcal{L}(\boldsymbol{\theta}^{\text{PS}}; \mathcal{D}_b) = 0,
\]
which implies:
\[
\nabla_{\boldsymbol{\theta}} \mathcal{L}(\boldsymbol{\theta}^{\text{PS}}; \mathcal{D}_b) 
= -\frac{p_a^{\text{PS}}}{p_b^{\text{PS}}} \nabla_{\boldsymbol{\theta}} \mathcal{L}(\boldsymbol{\theta}^{\text{PS}}; \mathcal{D}_a).
\]
Substituting yields:
\[
\nabla_{\boldsymbol{\theta}} P(\boldsymbol{\theta}^{\text{PS}}) 
= 2 p_a^{\text{PS}} p_b^{\text{PS}} \Delta _{\mathcal{L}}^{\text{PS}}\cdot \boldsymbol{v}_{\text{fair}}.
\]
So,
\[
\left. \frac{d\boldsymbol{\theta}_\rho}{d\rho} \right|_{\rho = 0} 
= -2 p_a^{\text{PS}} p_b^{\text{PS}} \Delta _{\mathcal{L}}^{\text{PS}} \cdot H^{-1} \boldsymbol{v}_{\text{fair}}.
\]

Next, we compute the derivative of the loss disparity:
\[
\left. \frac{d}{d\rho} \Delta \mathcal{L}(\boldsymbol{\theta}_\rho) \right|_{\rho = 0}
= \boldsymbol{v}_{\text{fair}}^\top \left. \frac{d\boldsymbol{\theta}_\rho}{d\rho} \right|_{\rho = 0}
= -2 p_a^{\text{PS}} p_b^{\text{PS}} \Delta _{\mathcal{L}}^{\text{PS}} \cdot 
\boldsymbol{v}_{\text{fair}}^\top H^{-1} \boldsymbol{v}_{\text{fair}}.
\]

Finally, apply a first-order Taylor expansion:
\[
\Delta_{\mathcal{L},\text{fair}}^{\text{PS}}(\rho) = 
\Delta _{\mathcal{L}}^{\text{PS}} + \rho \left. \frac{d}{d\rho} \Delta \mathcal{L}(\boldsymbol{\theta}_\rho) \right|_{\rho = 0} 
+ \mathcal{O}(\rho^2),
\]
\[
= \Delta _{\mathcal{L}}^{\text{PS}}
- 2\rho \, p_a^{\text{PS}} p_b^{\text{PS}} \Delta _{\mathcal{L}}^{\text{PS}} \cdot 
\boldsymbol{v}_{\text{fair}}^\top H^{-1} \boldsymbol{v}_{\text{fair}} 
+ \mathcal{O}(\rho^2),
\]
which completes the proof.
\end{proof}

\subsection{Discussion on Performative Optimal (PO) Solutions}

The PO solution is defined as 
\begin{equation}
    \boldsymbol{\theta}^{\text{PO}} := \mathbb{E}_{Z \sim \mathcal{D}_{\boldsymbol{\theta}}} \ell(\boldsymbol{\theta};Z),
\end{equation}
$\mathcal{D}_{\boldsymbol{\theta}}$ is the fixed point distribution.

At a high level, the PO solution and the approximation between PS and PO solution largely follows the analysis in \citep{perdomo_performative_2021, brown2020performative}, where Theorem 6 in \citep{brown2020performative} states that if conditions in Lemma \ref{lemma:SDPP} and $\ell(\cdot; \cdot)$ is $L_z$-Lipschitz in the second argument, then the PS and PO solution satisfies
\begin{equation*}
    \| \boldsymbol{\theta}^{\text{PO}} - \boldsymbol{\theta}^{\text{PS}} \|_2 \leq \frac{2 L_z \epsilon}{\gamma (1 - \epsilon)}.
\end{equation*}

We note that the corresponding proof and the bound both only depend on the strong convexity coefficient $\gamma$ and sensitivity coefficient $\epsilon$, where the fair objective and the original objective have the same values, but the fair objective did change the Lipschitz coefficient to $\Tilde{L}_z:=(1+\rho\overline{\ell}^2) L_z$ and $\Tilde{L}_z:=(1+\rho\overline{\ell}) L_z$ when using fair penalty and fair reweighting mechanisms, respectively. Therefore, as long as we replace $L_z$ with the corresponding new value $\Tilde{L}_z$, we can use the same proof steps to show that
\begin{equation*}
    \| \boldsymbol{\theta}^{\text{PO}}_{\textsf{fair}} - \boldsymbol{\theta}^{\text{PS}}_{\textsf{fair}} \|_2 \leq \frac{2 \Tilde{L}_z \epsilon}{\gamma (1 - \epsilon)},
\end{equation*}
where 
\begin{equation*}
    \boldsymbol{\theta}^{\text{PO}}_{\textsf{fair}} := \arg \min_{\boldsymbol{\theta}} \mathcal{L}_{\textsf{fair}}(\boldsymbol{\theta}; \mathcal{D}_{\boldsymbol{\theta}})
\end{equation*}

But as we show previously, the fairness metrics are only stable and meaningful in PP when measured at the PS and Fair-PS solutions, so there isn't much practical value in finding the FPO solution.

\end{document}